\documentclass[journal]{IEEEtran}
\IEEEoverridecommandlockouts
\usepackage{caption}
\usepackage{amsmath,amssymb,amsfonts}
\usepackage{amsthm}
\usepackage{graphicx}
\usepackage{textcomp}
\usepackage[pscoord]{eso-pic}
\usepackage{xcolor}
\usepackage{subfigure}
\usepackage{diagbox}
\usepackage{easymat}
\usepackage{listings}
\usepackage{booktabs}
\newtheorem{theory}{Theorem}
\usepackage{algorithm}  
\usepackage{algpseudocode}  
\usepackage{amsmath}  
\usepackage{multirow}
\newlength{\halfpagewidth}
\divide\halfpagewidth by 2

\def\BibTeX{{\rm B\kern-.05em{\sc i\kern-.025em b}\kern-.08em
    T\kern-.1667em\lower.7ex\hbox{E}\kern-.125emX}}
\begin{document}



\title{Multilayer Perceptron Based Stress Evolution Analysis under DC Current Stressing \\ for Multi-segment Wires}



\author{Tianshu Hou, Peining Zhen, Ngai Wong, Quan Chen, Guoyong Shi, Shuqi Wang, Hai-Bao Chen
\thanks{This work is supported in part by the National Key Research and Development Program of China under grant 2019YFB2205005, and in part by the Nature Science Foundation of China (NSFC) under No. 62034007. Corresponding author: Hai-Bao Chen.}
\thanks{Tianshu Hou, Peining Zhen, Guoyong Shi, Shuqi Wang and Hai-Bao Chen are with the Department of Micro/Nano Electronics, Shanghai Jiao Tong University. Ngai Wong is with the Department of Electrical and Electronic Engineering, University of Hong Kong. Quan Chen is with the School of Microelectronics, Southern University of Science and Technology.}}

\maketitle

\begin{abstract}
Electromigration (EM) is one of the major concerns in the reliability analysis of very large scale integration (VLSI) systems due to the continuous technology scaling. Accurately predicting the time-to-failure of integrated circuits (IC) becomes increasingly important for modern IC design. However, traditional methods are often not sufficiently accurate, leading to undesirable over-design especially in advanced technology nodes. In this paper, we propose an approach using multilayer perceptrons (MLP) to compute stress evolution in the interconnect trees during the void nucleation phase. The availability of a customized trial function for neural network training holds the promise of finding dynamic mesh-free stress evolution on complex interconnect trees under time-varying temperatures. Specifically, we formulate a new objective function considering the EM-induced coupled partial differential equations (PDEs), boundary conditions (BCs), and initial conditions to enforce the physics-based constraints in the spatial-temporal domain. The proposed model avoids meshing and reduces temporal iterations compared with conventional numerical approaches like FEM. Numerical results confirm its advantages on accuracy and computational performance.    
\end{abstract}


\section{introduction}
Electromigration (EM) reliability analysis has become a significant design consideration in very large scale integration (VLSI) systems due to the escalating current densities in interconnects resulted from technology scaling \cite{Warnock-5981968}. EM-induced voiding processes cause the increase of interconnect resistance, leading to degradation or potential destruction of circuit functionalities. Thus, it is important to develop accurate and efficient EM effect failure assessment methods for VLSI chips in 7-nm technology and below. However, the traditional Black's model \cite{Black1969} and Blech's effect model \cite{blech1976} target stress evolution prediction on single metal wires only, causing high prediction errors and excessive design margins. 


Several physics-based methods have been proposed recently~\cite{chatterjee2018:TCAD,chen2020:tcad,chen2015:aspdac,chen2015:dac,Sukharev2016:TDMR}. 
It should be noted that the major challenge of EM analysis in the physics-based methods is to solve stress diffusion induced partial differential equations (PDEs) governed by Korhonen's equation~\cite{Korhonen1993} with complex boundary conditions (BCs) and initial conditions.
In \cite{XHuang2014:DAC} and \cite{huang2016:tcad}, a new physics-based approach was proposed for EM assessment in power delivery networks of VLSI. The approach extended the reliability analysis of single metal wires to multi-segment interconnects to obtain the projected steady-state stress \cite{Sun2016:ACM}.
The mesh-based numerical methods such as the finite difference method (FDM) and the finite element method (FEM) can solve the PDEs arising from complex on-chip interconnect topologies but require a significant number of unknown variables due to the spatial and temporal discretization.
Moreover, an analytical solution of stress evolution for simple multi-branch interconnect trees was developed to predict dynamic stress evolution during the void nucleation phase under time-varying temperature \cite{chen2016:CDICS,Chen2017:TDMR}. The method constructed a basis function based on Laplace transformation, which provides new insights for EM reliability analysis. In \cite{Chen2019:VLSI}, the proposed method modified the accelerated separation of variables (ASOV) for describing dynamic stress evolution under constant temperature. 

\begin{figure}[t]
	\centering
	\includegraphics[width=0.8\columnwidth]{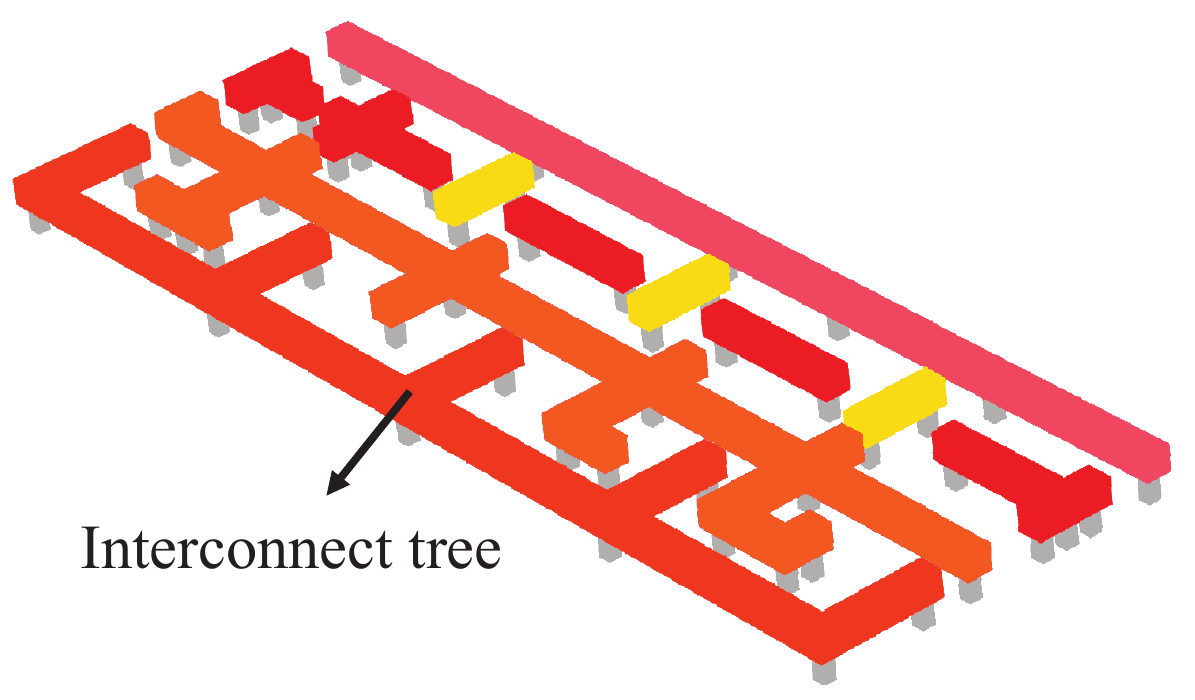}
	\caption{3-D schematic of a metal layer with complex interconnect tree structures in designing IC.}
	\label{fig:ipg}
\end{figure}

On the other hand, machine learning methods have demonstrated their capability to explore the invisible correlation of massive data in recent years~\cite{cr2020:neur,justin2018:jcp,diss1994,lagaris2000:tnn}. 
The algorithm \cite{IE1998:NN} employed a neural network as a nonlinear component of an appropriate trial function to solve PDEs with low demand on memory space. 
Physics-informed neural networks (PINN) encoded laws of physics into a neural network for discovering solutions of general PDEs \cite{PINN2019:Journal}.
Frameworks such as weak adversarial networks (WAN) and multi-fidelity physics-informed neural networks (MPINN) have been proposed for classical problems in fluids, quantum mechanics, reaction-diffusion systems, and the propagation of nonlinear shallow-water waves \cite{MENG2020:jcp, ZANG2020:jcp}. Inspired by recent progress with learning-based methods for solving PDEs, a data-driven meshless 2-D analysis method was proposed to calculate electric potential and electric field in VLSI interconnects \cite{Jin2021:DATE}.
The stress solution of the coupled EM-induced PDEs is non-smooth and depends on the structure of interconnect, shown in Fig.~\ref{fig:ipg}.  
One limitation of employing these competing learning-based schemes in EM analysis is that the prediction accuracy will decrease as the number of segments increases since the methods focus on solving a single PDE and cannot directly provide a global approximation for the interconnects governed by coupled PDEs subject to complex BCs. 
{\color{black}To mitigate this problem, our previous work \cite{hou:tcad} extended PINN to a new space-time physics-informed neural network (STPINN) for analyzing the EM-induced stress evolution by coupling the physics-based EM analysis with dynamic temperature incorporating Joule heating and via effect.} 


In this paper, we propose a new approach to achieve stress evolution solutions under time-varying temperatures during the void nucleation phase. The proposed method employs multilayer perceptrons (MLP) to generate differentiable, closed stress solutions on arbitrary complex multi-segment interconnect structures without a mesh generation. Furthermore, we compute the stress evolution motivated by dynamic temperature and analyze the kinetic difference against the constant temperature.
The proposed method is compared with PINN \cite{PINN2019:Journal}, FEM \cite{Comsol} and EMSpice \cite{sun2020:tdmr} in accuracy and performance. EMSpice is a simulation tool for full-chip EM analysis, which can obtain stress solutions of straight multi-segment interconnect trees during the void nucleation phase.
The proposed method shows high accuracy and computational savings. The main contributions of this paper are:

\begin{itemize}
	\item We propose a fast learning-based stress evolution computation method aiming at complex multi-segment interconnect structures. The method is based on MLP and requires no prior knowledge of stress evolution during the training process. Unlike numerical methods such as FDM and FEM, the proposed method, which is mesh-free, can obtain the EM-induced stress at a certain space and time without solving solutions at all meshing points. 
	
	\item We propose a new method to formulate the objective function to consider the constraints of EM induced-stress evolution consisting of the diffusion process, BCs, and initial conditions. 
	{\color{black}In stress analysis on multi-segment interconnects, compared with the state-of-the-art learning-based PINN method and STPINN, the proposed method is extended to reduce the demand on the number of training data and to achieve higher prediction accuracy at a shorter training time.}
	



	\item {\color{black}The proposed method can obtain the stress distribution of any complex multi-segment interconnect structure whose junction is connected to more than two adjacent segments. Transient stress evolution at any given aging time and location can be inferred by the proposed method. The advantages of the proposed method on accuracy and computational performance are verified by numerical results.
	}


\end{itemize}


The rest of the paper is organized as follows. Section~\ref{gradient-based} reviews the EM physics and the physics-based stress modeling. Section~\ref{optimization} generalizes the constrained problem according to the EM stress modeling and shows how to formulate the objective function. Section~\ref{method} introduces the framework of the proposed method and extends it to the dynamic model for time-varying temperature. Section~\ref{results} shows the results of the proposed method and performance comparison against competing methods. Section~\ref{conclusion} concludes this paper.

\section{EM Physics and Physics-based Stress Modeling}\label{gradient-based}

EM is the mass transport resulting from the momentum exchange between conducting electrons and metal atoms within the high-density current. 
In a dual-damascene structure, the metal atoms are subject to a mechanical driving force and an opposite electronic wind force, leading to a depletion at the cathode and an accumulation at the anode of metal wire. 
{\color{black}In this process, voids and hillocks are generated by the lasting electrical load, which develops a stress gradient along the metal wire.} 
Tensile stress promotes the formation of atom depletion and causes void nucleation when its value exceeds the critical value, defined as $\sigma_{crit}$. The copper atoms are blocked from diffusing towards inter-layer (ILD) and inter-metal dielectrics (IMD) by a barrier layer. Figs.~\ref{fig:upstream}~\&~\ref{fig:downstream} show the EM effect in the single copper metallization for electrons moving upward and downward, also referred to upstream electron flow and downstream electron flow \cite{li2015:tvlsi}. The void nucleation phase can be governed by the kinetics equation and the electrical resistance of interconnects is degenerated due to void growth after the nucleation phase \cite{Korhonen1993,Sukharev2008:TCAD}. 
\begin{figure}[t]
	\centering 
	\subfigure[]{
		\includegraphics[width=0.45\columnwidth]{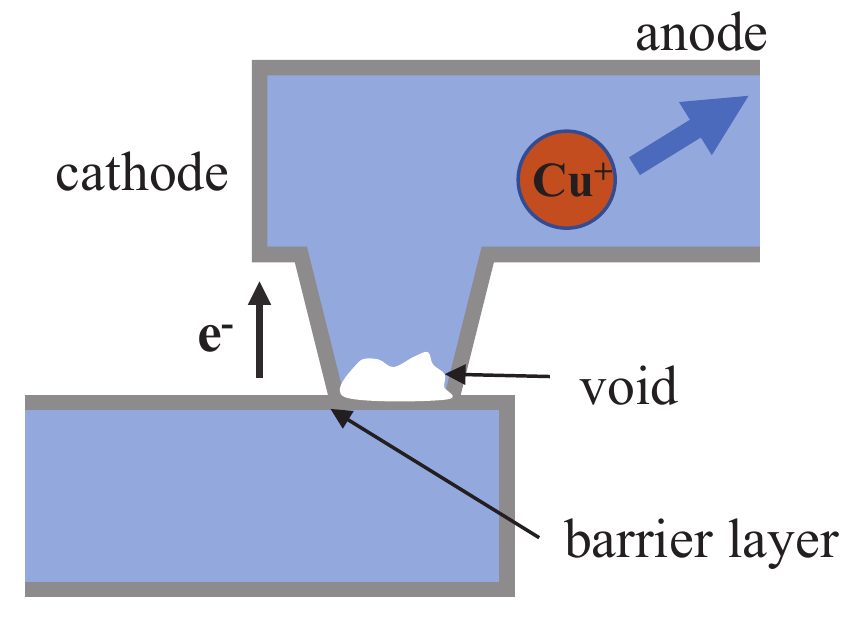}
		\label{fig:upstream}}
	\subfigure[]{
		\includegraphics[width=0.45\columnwidth]{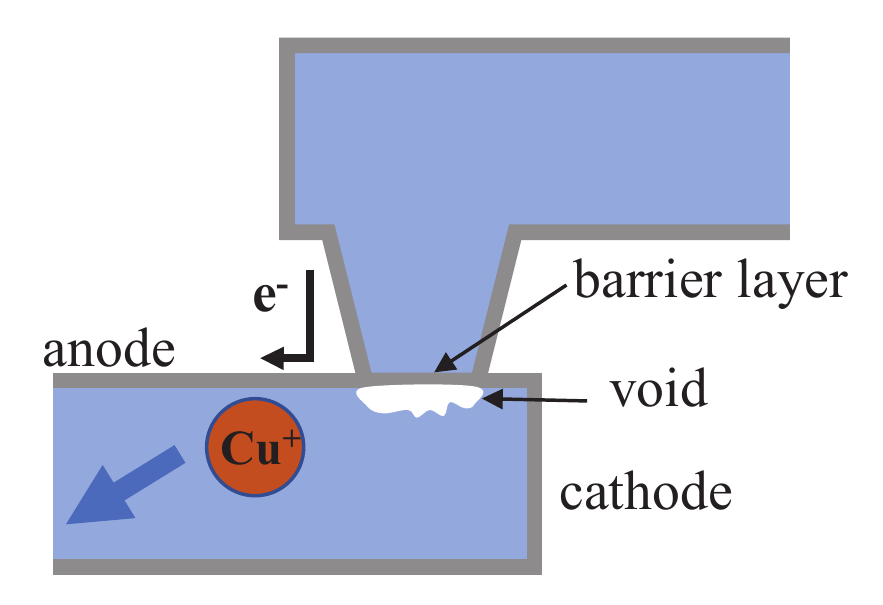}
		\label{fig:downstream}}
	\caption{EM effect under: (a) upstream electron flow and (b) downstream electron flow.}
	\label{fig:multi}
\end{figure}
The interconnect trees which consist of connected metal segments are bounded by the barrier layers at vias in VLSI chips, which results in the blocked region constraint for the metal atoms free-flowing. 

The Korhonen's equation describes the hydrostatic stress evolution $\sigma_i(x,t)$ in the $i$-th segment by diffusion-like equation \cite{Korhonen1993}, which takes the form 
\begin{equation}\label{eq:Korhonen's PDE}
\frac{\partial\sigma_i(x,t)}{\partial t}=\frac{\partial}{\partial x}\Big[\kappa_i\Big(\frac{\partial \sigma_i(x,t)}{\partial x}+G_{i}\Big)\Big],\\
\end{equation}
where $x$, $t$, $\kappa_i=D_aB\Omega/(kT)$ are location, time and stress diffusivity. It is supposed that the stress diffusivity is the same in each segment. The notation $B$ is the effective bulk related to line geometry, especially width, aspect, and grain morphology \cite{Dwyer2010:JAP,Hau:2000MRS}. The Boltzmann constant and the absolute temperature are $k$, $T$. The effective atomic diffusion coefficient, defined as $D_a=D_0\exp(-E_a/(kT))$, is typically determined by the interfacial and grain boundary diffusivities induced by grain microstructure development \cite{Dwyer2008:JAP}, where $D_0$ is the self-diffusion coefficient. Notations $\Omega$, $j$, $\rho$, $Z^*$ and $E_a$ represent the atomic lattice volume, current density, metal resistivity, the effective charge number and activation energy, respectively. The EM driving force is written as $G=|Z^*|e\rho j/\Omega$. 


We suppose there's no pre-existing residual stress along the interconnects and the IC is defined as
\begin{equation}\label{eq:ic}
\sigma_i(x,0)=0.
\end{equation}
The spatial gradient of stress evolution ($\partial \sigma/\partial x$) in the terminals and interior junctions of the interconnect tree are restrained by the BCs.
Specifically, BC at the terminal describes that the atomic flux is blocked at the terminal of the confined metal wire and the atomic flux is defined as
\begin{equation}\label{eq:atomic flux}
J(x,t)=\frac{D_aC_v\Omega}{ kT}\Big(\frac{\partial \sigma_r(x,t)}{\partial x}\Big|_{x=x_r}+G_r\Big),
\end{equation}
where $C_v$ is the number of metal atoms per unit volume and $r$ represents the terminal and interior junctions of the interconnect tree. During the nucleation phase of stress evolution, the atomic flux at terminals is equivalent to zero, then BC at terminals is expressed as
\begin{equation}\label{eq:nbc}
\kappa_b\Big(\frac{\partial \sigma_b(x,t)}{\partial x}\Big|_{x=x_b}+G_b\Big)=0,\\
\end{equation}
where $b$ represents the blocked terminals.

Within the interconnect trees, metal atoms diffuse across adjacent segments through interior junctions of the interconnect. 
In this way, the stress at the interior junction follows the flux conservation as
\begin{align}
&\sum_iw_{i,r}\cdot\kappa_i\Big(\frac{\partial\sigma_{i}(x,t)}{\partial x}\Big|_{x=x_r}+G_i\Big)\cdot n_{i,r}=0,\label{eq:fc}
\end{align}
where $w_{i,r}$ is the branch width of segment $i$ connected to interior junction $r$ and $n_{i,r}$ represents the unit normal direction of the junction $r$ on segment $i$, which is $+1$ for the left, below segments and $-1$ for the right, upper segments. 
The length and width of each segment are not exactly the same due to the design requirements. 
Furthermore, stress continuity condition shows that stress at interior junctions is continuous on the adjacent segments, which can be expressed as
\begin{equation}\label{eq:sc}
\sigma_{i_1}(x_i,t)=\cdots=\sigma_{i_m}(x_i,t),\\
\end{equation}
where segments $i_1,\cdots,i_m$ intersect at $x_i$. To this end, \eqref{eq:nbc}, \eqref{eq:fc} and \eqref{eq:sc} describe the BCs of \eqref{eq:Korhonen's PDE}.


\section{Gradient-based Analysis in Stress Evolution }\label{optimization}

In order to solve the stress evolution equations, we generalize the diffusion constrained problem and the gradient constrained problem in stress modeling.
The diffusion constrained problem focuses on discovering the stress diffusion process within each segment governed by Korhonen's equation \eqref{eq:Korhonen's PDE}, 
while the gradient constrained problem aims at finding the proper spatial gradient of stress at nodes to satisfy BCs.
In this section, we first solve the diffusion constrained problem by constructing a trial function. The physics-based constraints are then transformed into the gradient constraint.
After that, we formulate a gradient-based objective function for neural network training to perform the stress evolution analysis.
\subsection{Diffusion Constrained Problem}
In EM analysis, the diffusion-like Korhonen's equation \eqref{eq:Korhonen's PDE} constructs the diffusion constraint for stress evolution distribution within each segment in the time range $t\in(0,T_{steady}]$, where $T_{steady}$ is the upper limit of the observation time sufficient for EM evaluation or reaching the steady state. To satisfy the constraint, we introduce a trial function $\varPsi_t(x,t,L,\theta^-,\theta^+)$ as the solution of stress modeling. The function takes location $x$, time instance $t$, length of the segment $L$ and adjustable parameters $\theta^+,\theta^-$ as inputs, and is capable of providing solutions respecting the diffusion constraint subject to the initial condition \eqref{eq:ic}, which satisfies
\begin{equation}\label{eq:physics constraint}
\begin{aligned}
\varPsi_t(x_{i,j}&,t,L_i,\theta_{i}^-,\theta_{i}^+)=\mathop{\arg\min}_{\varPsi_t}\\
\Big\{&\sum_{i=0}^M\sum_{j=0}^N\Big|\varPsi_t(x_{i,j},0,L_i,\theta_{i}^-,\theta_{i}^+)\Big|^2\\+&
\sum_{i=0}^M\sum_{j=0}^N\Big|\frac{\partial}{\partial x_{i,j}}\Big[\kappa_i\Big(\frac{\partial \varPsi_t(x_{i,j},t_{i,j},L_i,\theta_{i}^-,\theta_{i}^+)}{\partial x_{i,j}}+G_{i}\Big)\Big]\\&\qquad\qquad-\frac{\partial\varPsi_t(x_{i,j},t_{i,j},L_i,\theta_{i}^-,\theta_{i}^+)}{\partial t_{i,j}}\Big|^2\Big\}.\\
\end{aligned}
\end{equation}
Here, $i$ represents the number of the segment. The notations $x_{i,j}\in(0,L_i),\ t_{i,j}\in(0,T_{steady}]$ are the $j$-th spatial and temporal collocation points of the $i$-th segment obtained by random sampling schemes. {\color{black} This minimization in \eqref{eq:physics constraint} will fulfill the physics constraints in \eqref{eq:Korhonen's PDE} \&~\eqref{eq:ic}.} {\color{black}The details for deriving the trial function are given in Appendix A.}
The trial function follows
\begin{equation}\label{eq:convolution}
	\begin{aligned}
	&\varPsi_t(x,t,L,\theta^-,\theta^+)\approx\\
	&\sum_{n=0}^{2}\Big(\frac{-dk(t,\theta^-)}{dt}\ast \big(g(\xi_1(n,x,L),t)+g(\xi_3(n,x,L),t)\big)\\
	&-k(0,\theta^-)\times\big(g(\xi_1(n,x,L),t)+g(\xi_3(n,x,L),t)\big)
	\\&+\frac{dk(t,\theta^+)}{dt}\ast \big(g(\xi_2(n,x,L),t)+g(\xi_4(n,x,L),t)\big)\\
	&+k(0,\theta^+)\times \big(g(\xi_2(n,x,L),t)+g(\xi_4(n,x,L),t)\big)\Big),\\
	\end{aligned}
\end{equation}
where 
\begin{equation}\label{eq:notation}
\begin{aligned}
&\xi_1(n,x,L) =(2n+2)L-x,\\
&\xi_2(n,x,L) =(2n+1)L-x,\\
&\xi_3(n,x,L) =(2n)L+x,\\
&\xi_4(n,x,L) =(2n+1)L+x,\\
\end{aligned}
\end{equation}
and
\begin{equation} \label{eq:basisFunc}
g(x,t)=2\sqrt{\frac{\kappa t}{\pi}}e^{-\frac{x^2}{4\kappa t}}-x\times\texttt{erfc}\{\frac{x}{2\sqrt{\kappa t}}\}.
\end{equation}
Here, $k(t,\theta)$ is an adjustable time-related function and the temporal convolution follows $a(t) \ast b(t)=\int_0^t a(\tau)b(t-\tau)d\tau$. 
It should be noted that the trial function is subject to the following Neumann BCs when $t\geq 0$
\begin{equation}\label{eq:Boundary}
\begin{aligned}
&\frac{\partial \varPsi_t(x,t,L,\theta^-,\theta^+))}{\partial x}-k(t,\theta^-)=0, x=0,\\
&\frac{\partial \varPsi_t(x,t,L,\theta^-,\theta^+))}{\partial x}-k(t,\theta^+)=0, x=L,\\
\end{aligned}%
\end{equation}%
{\color{black}which are motivated by \eqref{eq:atomic flux}.} It can be observed from \eqref{eq:Boundary} that the adjustable functions $k(t,\theta^{-}),k(t,\theta^+)$ are equivalent to the spatial gradients of the trial function at nodes of each segment. The superscript $-/+$ is employed to distinguish the preceding and subsequent node. We define this time-related spatial gradient at nodes as the \textbf{stress gradient}. 
It demonstrates that for any segment $i$, the function $k(t,\cdot)$ approximates the stress gradient by the adjustable parameter $\theta_i^{-/+}$ corresponding to the preceding/subsequent node. 
Since the trial function \eqref{eq:convolution} satisfies the diffusion constraint, the optimization problem of the stress modeling has been reduced from the original diffusion-gradient constrained problem to a gradient constrained problem with respect to the adjustable parameters $\theta$. In the next section, we present a systematic method of computing $\theta$ to deal with the constrained problem in stress modeling.

\subsection{Gradient Constrained Problem}
\begin{figure}[t]
	\centering
	\includegraphics[width=0.7\columnwidth]{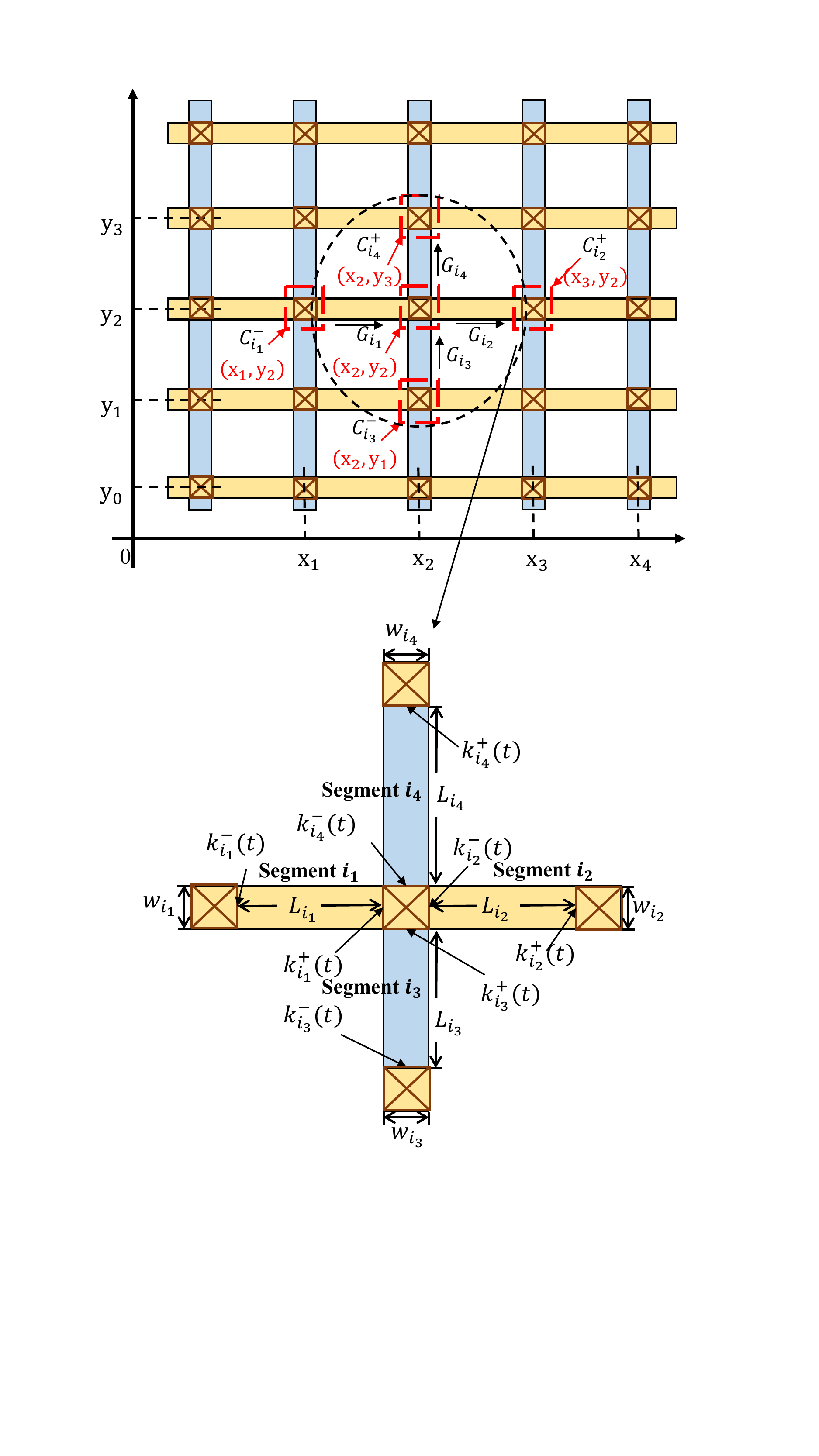}
	\caption{Notation definition for a cross-shaped interconnect tree.}
	\label{fig:ccs}
\end{figure}

To construct an accurate trial function for approximating stress evolution, we customize the expression of stress gradient evolution subject to the gradient constraint. The gradient constrained problem aims at adjusting the stress gradient to satisfy BCs.
We first illustrate the notation definition in Fig.~\ref{fig:ccs}. The Cartesian coordinate is employed to represent different positions in the interconnects with complex structures. The positive direction is used to distinguish the preceding and subsequent nodes of each segment. For a general interconnect tree, we use $L_i,\ w_i,\ G_i$ to describe the length, width, and EM driving force of the $i$-th segment. We denote $k(t,\theta_i^{-/+})=k_{i}^{-/+}(t)$ as the stress gradient at the preceding/subsequent node of the $i$-th segment, where $C_{i}^{-/+}$ represents the coordinate of the corresponding node.
Fig.~\ref{fig:ccs} shows an instance of cross-shaped interconnect with segments $i_1,i_2,i_3,i_4$. 
We define the center node of the interconnect as $C_i=C_{i_1}^+=C_{i_2}^-=C_{i_3}^+=C_{i_4}^-$ and the stress gradients at $C_i$ as
\begin{equation}\label{eq:km}
k_{i,m}(t)=\left\{
    \begin{aligned}
    &k_{i_1}^+(t),m=1,\\
    &k_{i_2}^-(t),m=2,\\
    &k_{i_3}^+(t),m=3,\\
    &k_{i_4}^-(t),m=4.\\
    \end{aligned}
    \right.
\end{equation}
We denote the collections of EM driving force and width in the adjacent segments connecting with $C_i$ by $\boldsymbol{G_i}=[G_{i_1},G_{i_2},G_{i_3},G_{i_4}]$ and $\boldsymbol{w_i}=[w_{i_1},w_{i_2},w_{i_3},w_{i_4}]$.
\begin{theory}\label{theory1} 
For arbitrary node $C_{i}$ connecting with segments $i_1,i_2,i_3,i_4$, 
the initial stress gradient at $C_i$ follows
\begin{equation}\label{eq:initial gradient}
\begin{aligned}
k_{i,m}(0)&=J(\boldsymbol{w_i},\boldsymbol{G_i})\\
&=\left\{
    \begin{aligned}
    &-G_b,\qquad\qquad\qquad\qquad\text{at terminal,}\\
    &\frac{w_{i_2}G_{i_2}+w_{i_4}G_{i_4}-w_{i_1}G_{i_1}-w_{i_3}G_{i_3}}{w_{i_1}+w_{i_2}+w_{i_3}+w_{i_4}},\\&\qquad\qquad\qquad\qquad\qquad\qquad m=1,3,\\
    &\frac{w_{i_1}G_{i_1}+w_{i_3}G_{i_3}-w_{i_2}G_{i_2}-w_{i_4}G_{i_4}}{w_{i_1}+w_{i_2}+w_{i_3}+w_{i_4}},\\&\qquad\qquad\qquad\qquad\qquad\qquad m=2,4.\\
    \end{aligned}
    \right.
\end{aligned}
\end{equation}
Here, $G_b$ is the EM driving force in the terminal segment.
\end{theory}
\begin{proof}\renewcommand{\qedsymbol}{}
	See Appendix B.\end{proof}

It demonstrates in Theorem \ref{theory1} that the initial stress gradient in \eqref{eq:convolution} can be derived by the adjacent EM driving forces and widths.
We suppose that there are $M$ segments connecting with the node $C_i$ and define the index of the adjacent segments as $\Gamma_i$. We employ a neural network with adjustable parameter $\alpha$ to obtain time derivatives of the first $M-1$ stress gradients
\begin{equation}\label{eq:mlp}
[\frac{dk_{i,\Gamma_i(1)}(t)}{dt},\cdots,\frac{dk_{i,\Gamma_i(M-1)}(t)}{dt}]
=F(t,C_i,\boldsymbol{G_i},\alpha).
\end{equation}
The neural network takes the time instance, the node coordinate $C_i$ as well as the adjacent EM driving forces $\boldsymbol{G}_i$ as inputs. In this way, we can employ the adjustable parameter $\alpha$ to obtain stress gradients on the whole interconnect instead of employing different parameters $\theta_{i}^{-/+}$ for stress gradient at each node of each segment.
We also define a transformation $H(\cdot)$ to obtain the stress gradient satisfying conditions \eqref{eq:nbc} \& \eqref{eq:fc}, which takes the form
\begin{equation}\small\label{eq:gt}
\begin{aligned}
\frac{dk_{i,m}(t)}{dt}&=H(\frac{dk_{i,\Gamma_i(1)}(t)}{dt},\cdots,\frac{dk_{i,\Gamma_i(M-1)}(t)}{dt})\\&=\left\{
    \begin{aligned}
    &0,\qquad\qquad\qquad\qquad\qquad\qquad\ \ \qquad\ \  \text{at terminal,}\\
    &\frac{dk_{i,m}(t)}{dt},\qquad\qquad\ \ \ \ \  m=\Gamma\!_i(1),\!\cdots\!,\Gamma\!_i(M-1),\\
    &\frac{1}{w_{i_m}}\sum_{\!j\!=\!\Gamma\!_i\!(\!1\!)}^{\Gamma\!_i\!(\!M\!-\!1\!)}\Big((-1)^{j\!+\!m\!+\!1}w_{i_j}\frac{dk_{i,j}(t)}{dt}\Big), m=\Gamma\!_i(M).\\
    \end{aligned}
    \right.
\end{aligned}
\end{equation}
The inputs of \eqref{eq:gt} are the outputs of \eqref{eq:mlp}. The transformation $H(\cdot)$ and the initial stress gradient calculation $J(\cdot)$ in Theorem \ref{theory1} satisfy the constraints corresponding to BC at terminals and the flux conservation.
Therefore, the trial function \eqref{eq:convolution} describing stress evolution prediction of the $i$-th segment can be rewritten as
\begin{equation}
\varPsi_t(x,t,\theta_{i}^-,\theta_{i}^+,L_i)=\varPsi_t(x,t,\mathbb{G}_i,\mathbb{W}_i,\mathbb{C}_i,L_i,\alpha),
\end{equation}
where $\mathbb{G}_i=[\boldsymbol{G_{i}^-},\boldsymbol{G_{i}^+}],\mathbb{W}_i=[\boldsymbol{w_{i}^-},\boldsymbol{w_{i}^+}],\mathbb{C}_i=[C_{i}^-,C_{i}^+]$ are the collections of adjacent EM driving forces, widths and coordinates of the preceding and subsequent nodes in the $i$-th segment.
The trial function will approximate the solution of stress modeling once it satisfies \eqref{eq:sc}. We can penalize the deviations of \eqref{eq:sc} by minimizing the loss
\begin{equation}\label{eq:loss}
\begin{aligned}
E[\alpha]=&\sum_{i=1}^{N_n}\sum_{j=1}^{M_i-1}\sum_{k=1}^{N_c}\Big| \varPsi_t(x_{i_j},t_{i,k},\mathbb{G}_{i_j},\mathbb{W}_{i_j},\mathbb{C}_{i_j},L_{i_j},\alpha)\\&-\varPsi_t(x_{i_{j+1}},t_{i,k},\mathbb{G}_{i_{j+1}},\mathbb{W}_{i_{j+1}},\mathbb{C}_{i_{j+1}},L_{i_{j+1}},\alpha) \Big|^2.
\end{aligned}
\end{equation}
Here, the $i_j$-th segment and the $i_{j+1}$-th segment intersect at the $i$-th interior junction node and we suppose the $i$-th node has $M_i$ adjacent segments. The location of the $i$-th node on the $i_j$-th segment is denoted by $x_{i_j}$, which follows $x_{i_j}=0$ at the preceding node and $x_{i_j}=L_{i_j}$ at the subsequent node. The notation $N_n$ represents the number of interior junction nodes in the interconnect tree.

The objective function \eqref{eq:loss} focuses on keeping stress continuous at interior junctions in arbitrary $N_c$ time instances. The competing learning-based methods (such as PINN) constrain neural networks using the diffusion equation and the corresponding BCs, initial conditions by the loss function, which will require a large number of training data for the loss function when the number of coupled diffusion equations increases. Thus, PINN cannot provide satisfactory accuracy when it is expanded to analyze stress evolution on large interconnect trees with multi segments. To solve stress modeling equations of large interconnects through neural networks, we propose a new objective function \eqref{eq:loss} enforcing the physics-based constraints at segment nodes, which is in low demand for training data compared with the training schemes requiring sampling collocation points in the whole interconnect. This low demand for training data stems from a preconstructed trial function as the solution of the stress modeling equation satisfying the diffusion constraint, BC at terminals, and flux conservation.
No additional derivative calculation is required during the training process in the proposed method compared with PINN. 
This results in significant computational savings in the training process.

The parameters in the trial function are updated by minimizing \eqref{eq:loss}. After convergence, we can obtain the stress evolution of the interconnect tree with any structure by the trial function.

\begin{figure}[tb]
	\centering
	\includegraphics[width=1.0\columnwidth]{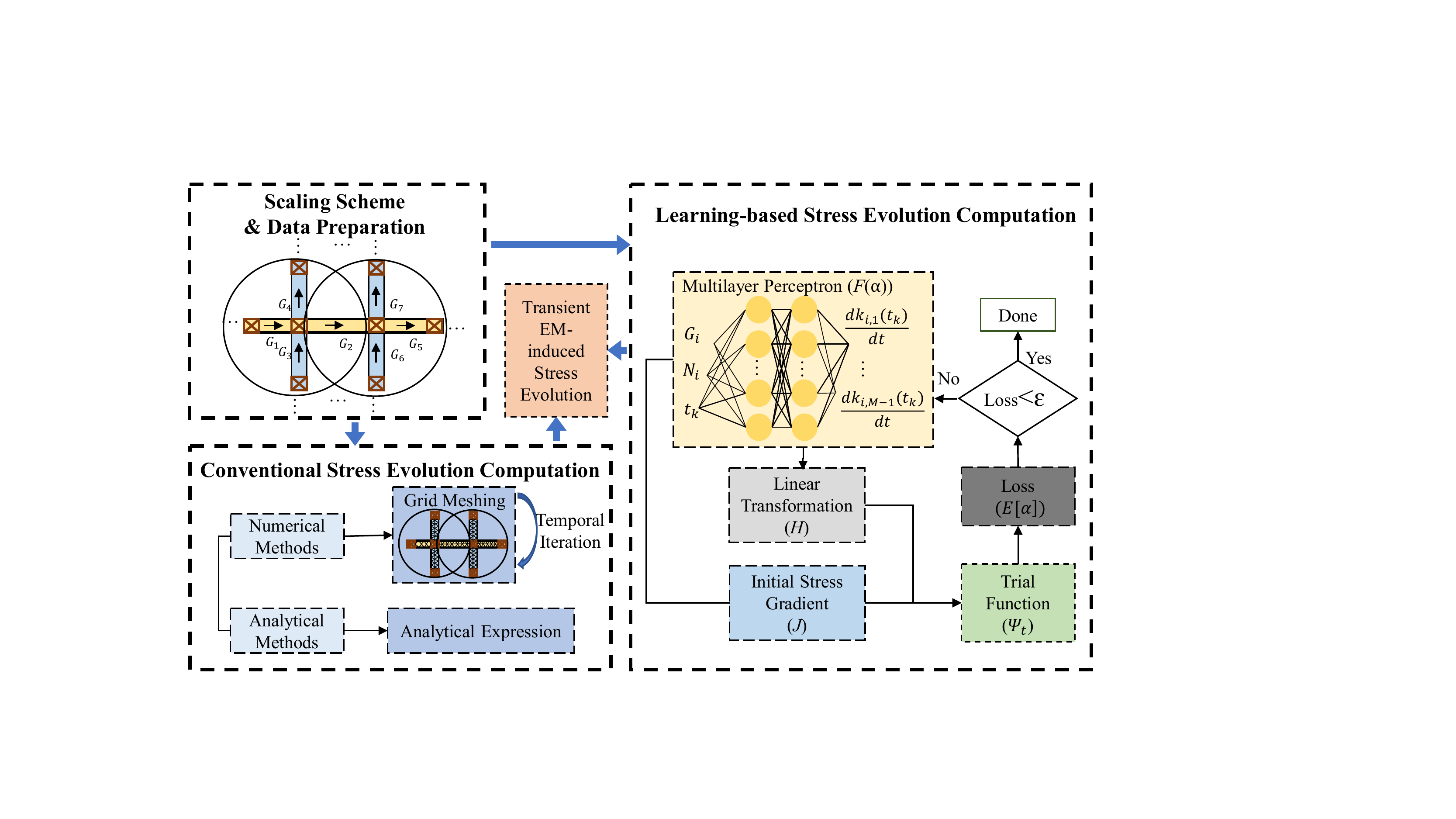}
	\caption{{\color{black}The flowcharts of the conventional methods and the proposed method in EM-induced stress evolution analysis.}}
	\label{fig:flowchart}
\end{figure}


\section{Multilayer Perceptron Method}\label{method}

In this section, we propose a multilayer perceptron based method to obtain the stress evolution in complex interconnect trees. {\color{black}Fig.~\ref{fig:flowchart} shows the flowcharts of the conventional EM stress computation and the proposed learning-based stress evolution computation. The conventional methods consist of numerical methods that employ grid meshing and temporal iterations to obtain the mesh-dependent stress evolution, and analytical methods which derive closed-forms for stress development on specific interconnect geometry. In the proposed scheme, we replace the linear span of a finite set of local basis functions in numerical methods such as FEM with nonlinear and linear operations in MLP to obtain mesh-free stress evolution with few temporal iterations.}
As shown in Fig.~\ref{fig:flowchart}, in the proposed method, we generate the training data and the test data according to interconnect tree structure and EM driving force after performing the scaling scheme and data preparation. The trial function which is related to the neural network $F$, the linear transformation $H$, and the initial stress gradient calculation $J$, provides the stress evolution prediction through the input data. The prediction accuracy will increase as the objective function converges to a global minimization.
More details in the flowchart will be discussed.

\subsection{Scaling Scheme and Data Preparation}
One potential limitation in stress modeling is that the related coefficients differ by a few orders of magnitude. 
{\color{black}For neural network training, we propose a scaling scheme to normalize the coefficients by rewriting Korhonen's equation \eqref{eq:Korhonen's PDE}}
\begin{equation}
\frac{\partial \sigma_{i}(x,t)}{\partial t}=\frac{\partial}{\partial x}\Big[\frac{\omega_x^2}{\omega_t}\kappa_{i,o}\Big(\frac{\partial \sigma_i(x,t)}{\partial x}+\frac{\omega_{\sigma}}{\omega_x}G_{i,o}\Big)\Big],\\
\end{equation}%
{\color{black}where $t=\omega_tt_o,x=\omega_xx_o,\ \sigma_i(x,t)=\omega_{\sigma}\sigma_{i,o}(x_o,t_o)$ are the scaled time, location and stress evolution, respectively. We redefine $G_i=\omega_{\sigma}/\omega_xG_{i,o},\ \kappa_i=\omega_x^2/\omega_t\kappa_{i,o}$ as the scaled EM driving force and stress diffusivity corresponding to the $i$-th segment, and $x_{o},t_{o},\sigma_{i,o}(x_o,t_o),G_{i,o},\kappa_{i,o}$ as the coefficients with original magnitude in the stress modeling.} With the constant scaling factors of location $\omega_x$, time $\omega_t$ and stress evolution $\omega_{\sigma}$, the scaling scheme takes the form
\begin{equation}\label{preprocess}
\begin{aligned}
\sigma_i(x,t,\kappa_i,G_i)&=\sigma_i(\omega_x
x_o,\omega_tt_o,\frac{\omega_x^2}{\omega_t}\kappa_{i,o},\frac{\omega_{\sigma}}{\omega_x}G_{i,o})\\&=\omega_{\sigma}\sigma_{i,o}(x_o,t_o,\kappa_{i,o},G_{i,o}).
\end{aligned}%
\end{equation}%
In this way, we regular the coefficients for the following neural network training.  
The scaling scheme is also effective for BCs and initial conditions. 
{\color{black}This results in taking $t,\ x,\ G_i,\ \kappa_i$ as the inputs for data preparation. It should be noted that the stress prediction $\sigma_i$ obtained by the scaled inputs should be restored to the original magnitude by the scaling factor $\omega_{\sigma}$.}

Then we generate the dataset in the data preparation. {\color{black}A collection of parameters is required to describe the geometry and EM driving force of each segment. For the $i_j$-th segment connecting with nodes $C_{i_j}^+,\ C_{i_j}^-$, the collection can be generalized as $\{\mathbb{G}_{i_j},\mathbb{W}_{i_j},\mathbb{C}_{i_j},L_{i_j}\}$, where $L_{i_j,\ }\mathbb{C}_{i_j}=[C_{i_j}^-,C_{i_j}^+]$ represent the length and the node coordinates of the segment, and $\mathbb{G}_{i_j},\mathbb{W}_{i_j}$ represent the EM driving forces and the widths of the adjacent segments connecting with the nodes. The breadth-first traversal method is employed for generating the above collection.

During the training phase, the required training dataset are the inputs of the loss function \eqref{eq:loss}, enforcing the stress continuity condition, which describes the relationship between the stress developments at the same node on the intersecting segments at arbitrary time instances.
The location of each node with respect to the intersecting $i_j$-th segment is denoted by $x_{i_j}$, which is equivalent to zero for the preceding node and $L_{i_j}$ for the subsequent node.
The collection of $N_c$ time instances $t_{i,k}(k=1,\cdots,N_c)$ is randomly sampled in the observation temporal range $(0,T_{steady}]$.
To this end, for the adjacent segments $i_j$ and $i_{j+1}$, we generalize the training data by
\begin{equation}\nonumber\small
\begin{aligned}
&\{x_{i_j},t_{i,k},\mathbb{G}_{i_j},\mathbb{W}_{i_j},\mathbb{C}_{i_j},L_{i_j};\\&
x_{i_{j+1}},t_{i,k},\mathbb{G}_{i_{j+1}},\mathbb{W}_{i_{j+1}},\mathbb{C}_{i_{j+1}},L_{i_{j+1}}\}.
\end{aligned}%
\end{equation}}%

{\color{black}During the inference phase, to obtain stress evolution in the $i_j$-th segment at the observation location and time $t$, the location should be transformed to the location with respect to the $i_j$-th segment satisfying $x\in [0,L_{i_j}]$.} The input of inference procedure follows $\{x,t,\mathbb{G}_{i_j},\mathbb{W}_{i_j},\mathbb{C}_{i_j},L_{i_j}\}$, which is the input collection for the trial function. {\color{black}It should be noted that although both the training and inference datasets are extracted from the same case with specific interconnect geometry and EM driving forces, the training dataset only includes collections of interior junction nodes, while the inference data can be collections describing arbitrary positions within the interconnect wire.}



\subsection{Learning-based Stress Evolution Computation}
The neural network, as the extension of mathematical models, has been developed for approaching solutions of PDEs recently. 
Moreover, the multilayer perception is proven to be a universal function approximator that applies linear and nonlinear transformations to inputs \cite{HORNIK1989359}. 
Since MLP is straightforward to understand and manipulate, in this work, we employ MLP for the nonlinear approximation $H$ in the trial function to solve the coupled EM-induced PDEs. 

In order to obtain solutions of the trial function \eqref{eq:convolution}, we employ the Gauss-Legendre quadrature algorithm to perform a fast convolution operation. We rewrite the trial function of the $i$-th segment
\begin{equation}\label{eq:final}\small
\begin{aligned}
	&\varPsi_t(x,t,\mathbb{G}_i,\mathbb{W}_i,\mathbb{C}_i,L_i,\alpha)\\
	&=\sum_{n=0}^{2}\sum_{j=0}^{N_g}\frac{A_jt}{2}\Big[-d_i^-(t_j^{+})\times\Big(g\big(\xi_1(n,x),t_j^{-}\big)+g\big(\xi_3(n,x),t_j^{-}\big)\Big)\\
	&+d_i^+(t_j^{+})\times\Big(g\big(\xi_2(n,x),t_j^{-}\big)+g\big(\xi_4(n,x),t_j^{-}\big)\Big)	\Big]\\
	&+\sum_{n=0}^{2}\Big[-k_i^-(0)\times\Big(g(\xi_1(n,x),t)+g(\xi_3(n,x),t)\Big)\\
	&+k_i^+(0)\times \Big(g(\xi_2(n,x),t)+g(\xi_4(n,x),t)\Big)\Big],\\
\end{aligned}
\end{equation}
where
\begin{equation}\small
\begin{aligned}
&t_j^{+}=\frac{t}{2}+\frac{t}{2}t_j,\ t_j^{-}=\frac{t}{2}-\frac{t}{2}t_j,\\
&k_i^+(0)=J(\boldsymbol{w}_i^+,\boldsymbol{G}_i^+),k_i^-(0)=J(\boldsymbol{w}_i^-,\boldsymbol{G}_i^-)\\
&d_i^+(t)=H(F(t,C_i^+,\boldsymbol{G}_i^+,\alpha)),d_i^-(t)=H(F(t,C_i^-,\boldsymbol{G}_i^-,\alpha))
\end{aligned}
\end{equation}
Here, $A_j,\ t_j$ are the Gaussian weights and zero points of Legendre polynomial. We use the convolution kernel for time discretizing, then substitute the discrete time into Gauss-Legendre integration in \eqref{eq:final}. Section~\ref{sec:perf} will analyze the impact of the number of discrete integration series ($N_g$) on accuracy. The objective function \eqref{eq:loss} is calculated according to the trial function \eqref{eq:final} employing the Gauss-Legendre quadrature algorithm and constrains the numerical relationship between solutions of the trial function with different specific inputs. This leads to MLP training based on known input data rather than prior knowledge of stress evolution. 
As the loss converges to a global minimum, the trial function can provide accurate stress approximations.

\subsection{Dynamic EM model under time-varying temperature}
\begin{figure}[tb]
	\centering
	\includegraphics[width=0.7\columnwidth]{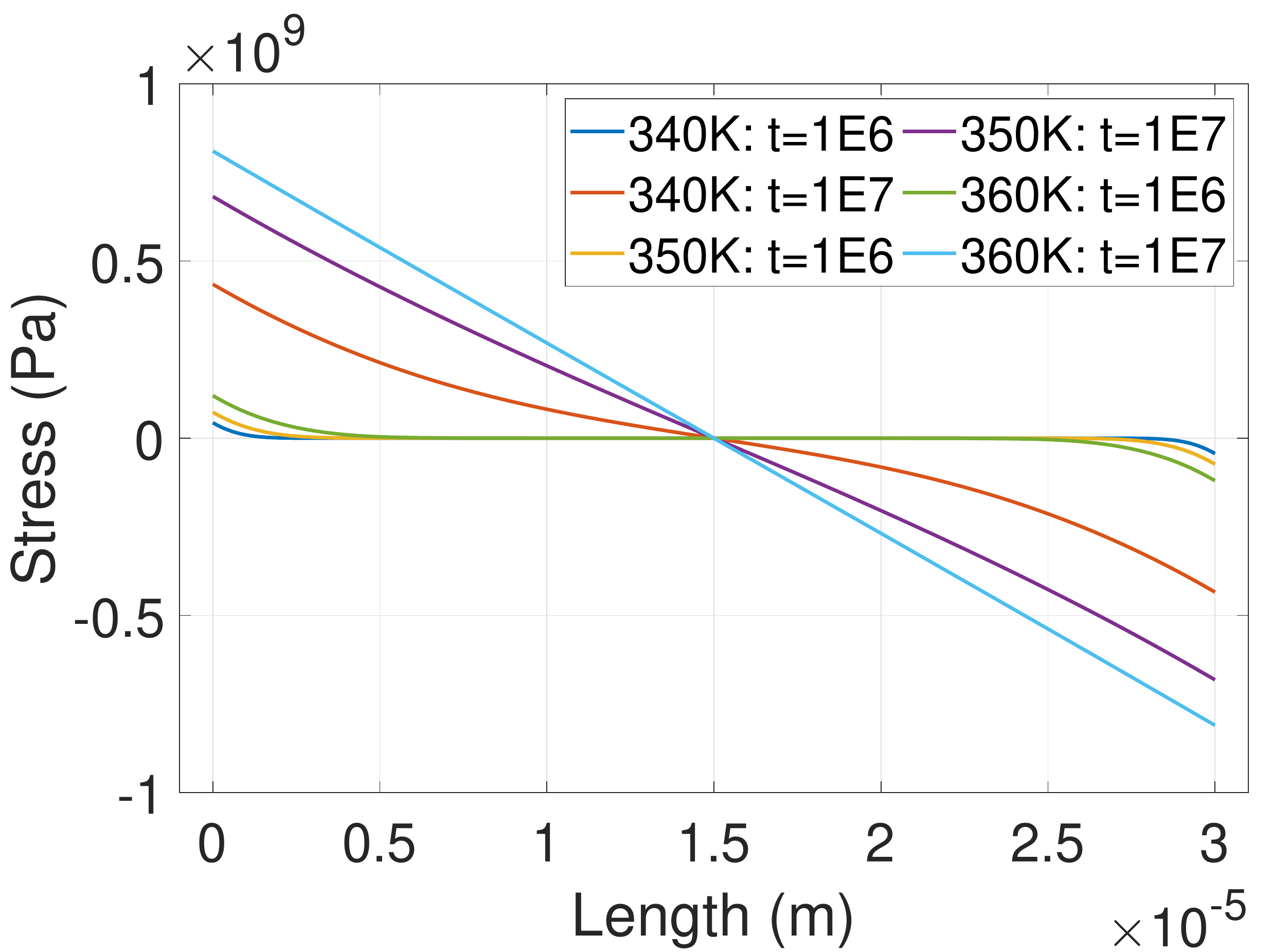}
	\caption{Stress evolution under varying temperature.}
	\label{fig:temp}
\end{figure}
It is shown in Korhonen's equation that the varying diffusivity caused by dynamic temperature will efficiently accelerate or decelerate the stress build-up \cite{Cook2018:TVLS}, shown in Fig.~\ref{fig:temp}. We assume that the temperature of interconnects is time-dependent and employ the parameter $\kappa(t)$ describing the stress diffusivity under time-varying temperature. Then we rewrite \eqref{eq:Korhonen's PDE}
\begin{equation}\label{eq:dynamic}
\frac{\partial\sigma}{\partial T'}=\frac{\partial}{\partial x}\Big[\kappa_0\Big(\frac{\partial \sigma}{\partial x}+G\Big)\Big],
\end{equation}
where $T'=\int_0^t(\kappa(t')/\kappa_0) dt'$ and $\kappa_0$ is a constant. It demonstrates that the analysis for stress evolution under time-varying temperature can be considered as solving EM-induced stress equation under constant temperature after a nonlinear transformation from the temporal variable $t$ to $T'$. We employ an MLP with one hidden layer to predict this nonlinear transformation and use the trial function with adjustable weights \eqref{eq:final} to predict the stress evolution under the constant temperature $\kappa_o$. The experimental results in Section~\ref{sec:ana} show the effectiveness of the dynamic EM model in obtaining time-varying temperature related stress evolution.
\section{results and discussions}\label{results}
In this section, we present stress evolution results obtained from the proposed method and validate its accuracy and performance under constant and time-varying temperatures. In the experiments, the constant temperature is set to be $350\ K$ and the time-varying temperature is set to be $[350+30\sin(4\times 10^{-8}\pi t)]\ K$. First, we analyze the impact of various widths of interconnect trees on stress evolution and expand the model to a dynamic model for EM analysis under time-varying temperature. Then, we discuss how the number of discrete integration and the size of training data affects prediction accuracy. We also test cases with different numbers of layers and neurons to customize the MLP structure. We verify the training acceleration of the proposed model compared with learning-based PINN. Finally, we prove the scalability of the proposed method for EM analysis on straight multi-segment interconnects and complex multi-segment interconnects. We compare the proposed method against the FEM \cite{Comsol}, PINN \cite{PINN2019:Journal}, EMSpice \cite{sun2020:tdmr} for accuracy, training time and computational time, respectively. The proposed method and PINN are implemented in Python 3.6.2 with Tensorflow 1.12.0 and the EMSpice is implemented in Python 3.6.2. The experiments are carried out on a Linux server with 2.20-GHz Xeon processors and NVIDIA GTX1080Ti. The FEM simulations are performed by COMSOL Multiphysics software \cite{Comsol} in 2-D structures.

In the training process, we employ the second-order based optimizer L-BFGS \cite{LBFGS1995} to adjust weights in MLP and employ $\tanh$ as the activation function. Scaling factors of length, time and stress are configured as $\omega_{x}=1\times 10^{-5},\ \omega_t=1\times 10^{-7},\ \omega_{\sigma}=1\times 10^{-7}$. The initial learning rate is $0.001$ with Xavier's initialization method. 
For the comparisons of the proposed method, FEM and EMSpice, the values of parameters used to calculate the stress evolution are set as $k=1.38\times 10^{-23}J/K,\ e=1.6\times10^{-19}C,\ Z^*=10,\ E_a=1.1eV,\ B=1\times10^{11}Pa,\ D_0=5.2\times10^{-5}m^2/s,\ \rho=2.2\times10^{-8}\Omega\cdot m,\ \Omega=8.78\times10^{-30}m^3,\ \sigma_{crit}=4\times10^{8}Pa$.
\subsection{Accuracy Analysis}\label{sec:ana}

\begin{figure}[t]
	\centering 
	\subfigure[]{
		\includegraphics[width=0.55\columnwidth]{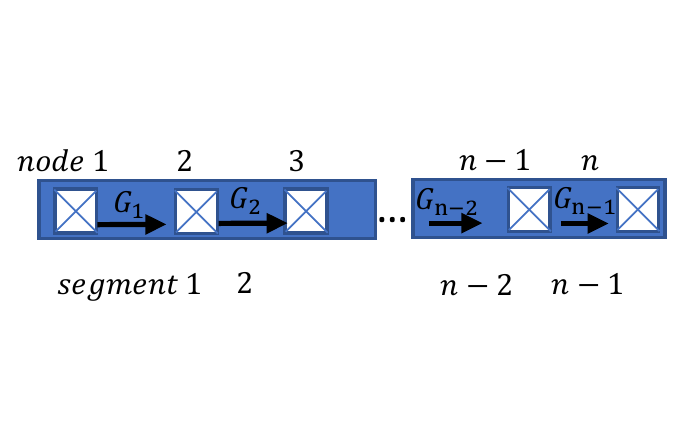}
	\label{fig:seg1d}}
		\subfigure[]{
		\includegraphics[width=0.35\columnwidth]{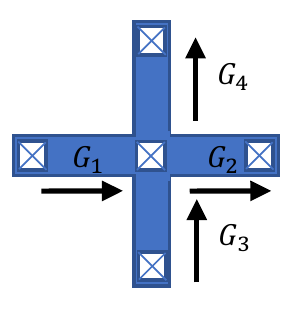}
	\label{fig:cro1d}}
	\caption{Interconnect tree structures: (a) the multi-segment straight wire; (b) the cross-shaped five-terminal wire.}
	\label{fig:multi}
\end{figure}

For validation, we obtain the stress evolution on the multi-segment straight wires and cross-shaped five-terminal wires, shown in Fig.~\ref{fig:multi}. We configure a four-segment interconnect tree of $L_1=10\ \mu m,\ L_2=20\ \mu m,\ L_3=10\ \mu m,\ L_4=10\ \mu m$ within current densities $j_1=4\times 10^9\ A/m^2,\ j_2=-1\times 10^9\ A/m^2,\ j_3=-4\times 10^9\ A/m^2,\ j_4=-1\times 10^9\ A/m^2$. We employ a 5-layer MLP with 50 neurons per layer to construct the trial function. The number of discrete series satisfies $N_g=8$ and the size of training data is set as $N_c=30$. Figs.~\ref{fig:1d1_1111} \&~\ref{fig:1d1_1221} show the stress evolution on segments with different widths and demonstrate that the results obtained from the proposed method fit well with the solutions obtained from FEM, with relative errors of $0.07\%,\ 0.06\%$.

In the simulation of a cross-shaped five-terminal wire, we configure trees of $L_1=20\ \mu m,\ L_2=10\ \mu m,\ L_3=20\ \mu m,\ L_4=30\ \mu m$ and the current densities are set to $j_1=4\times 10^9\ A/m^2,\ j_2=2\times 10^9\ A/m^2,\ j_3=1\times 10^9\ A/m^2,\ j_4=7\times 10^9\ A/m^2$. The comparisons of stress evolution are shown in Figs.~\ref{fig:2d1_1111} \&~\ref{fig:2d1_1221}. It demonstrates that compared with FEM, the proposed method achieves stress with $0.45\%,\ 0.91\%$ relative errors for the two cases.
From Fig.~\ref{fig:2d1}, we can observe the difference in stress evolution on interconnects with different segment widths. It illustrates the significance of considering the width of interconnects in the EM reliability problem.
\begin{figure}[t]
	\centering 
	\subfigure[]{
		\includegraphics[width=0.45\columnwidth]{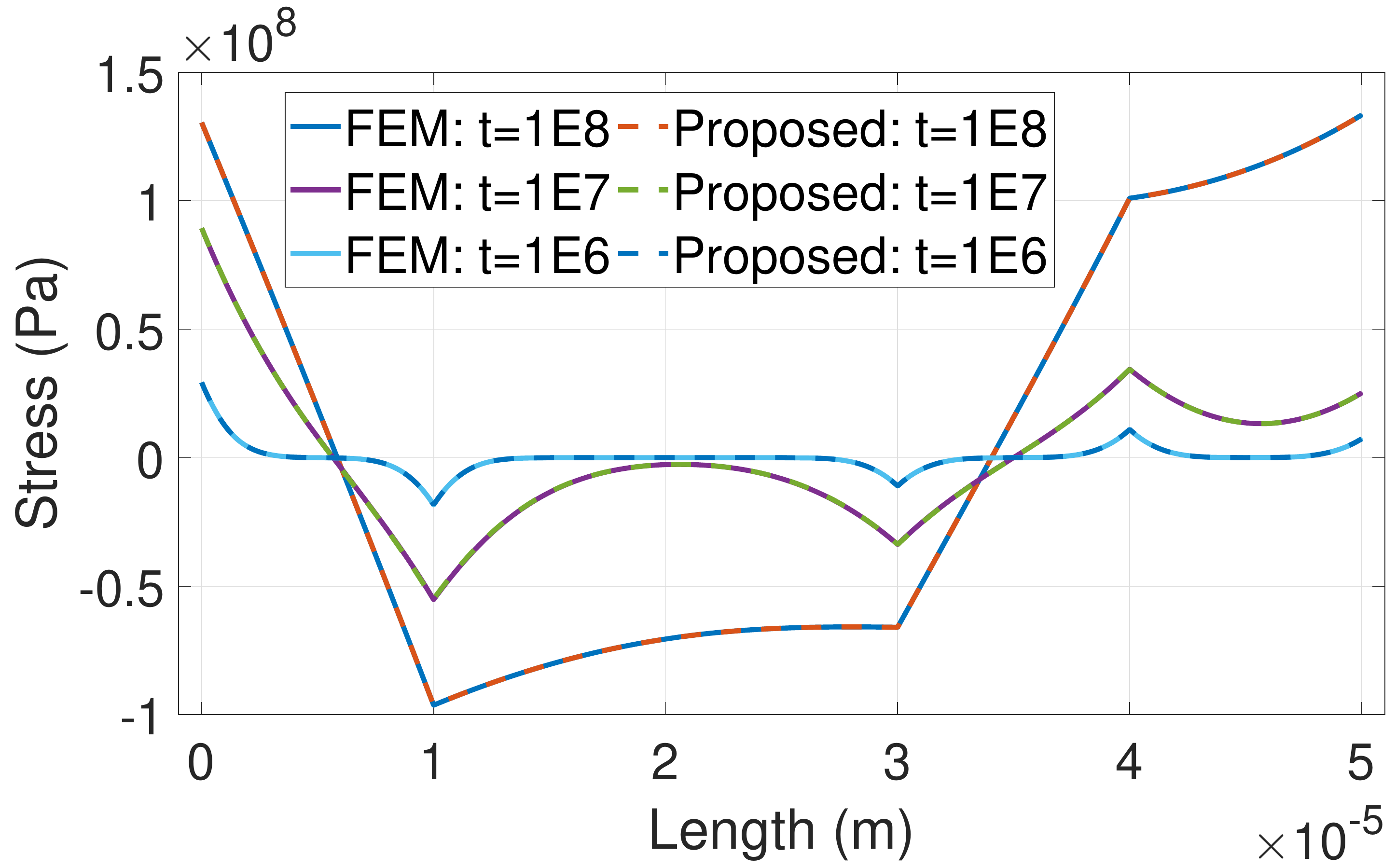}
	\label{fig:1d1_1111}}
	\subfigure[]{
		\includegraphics[width=0.45\columnwidth]{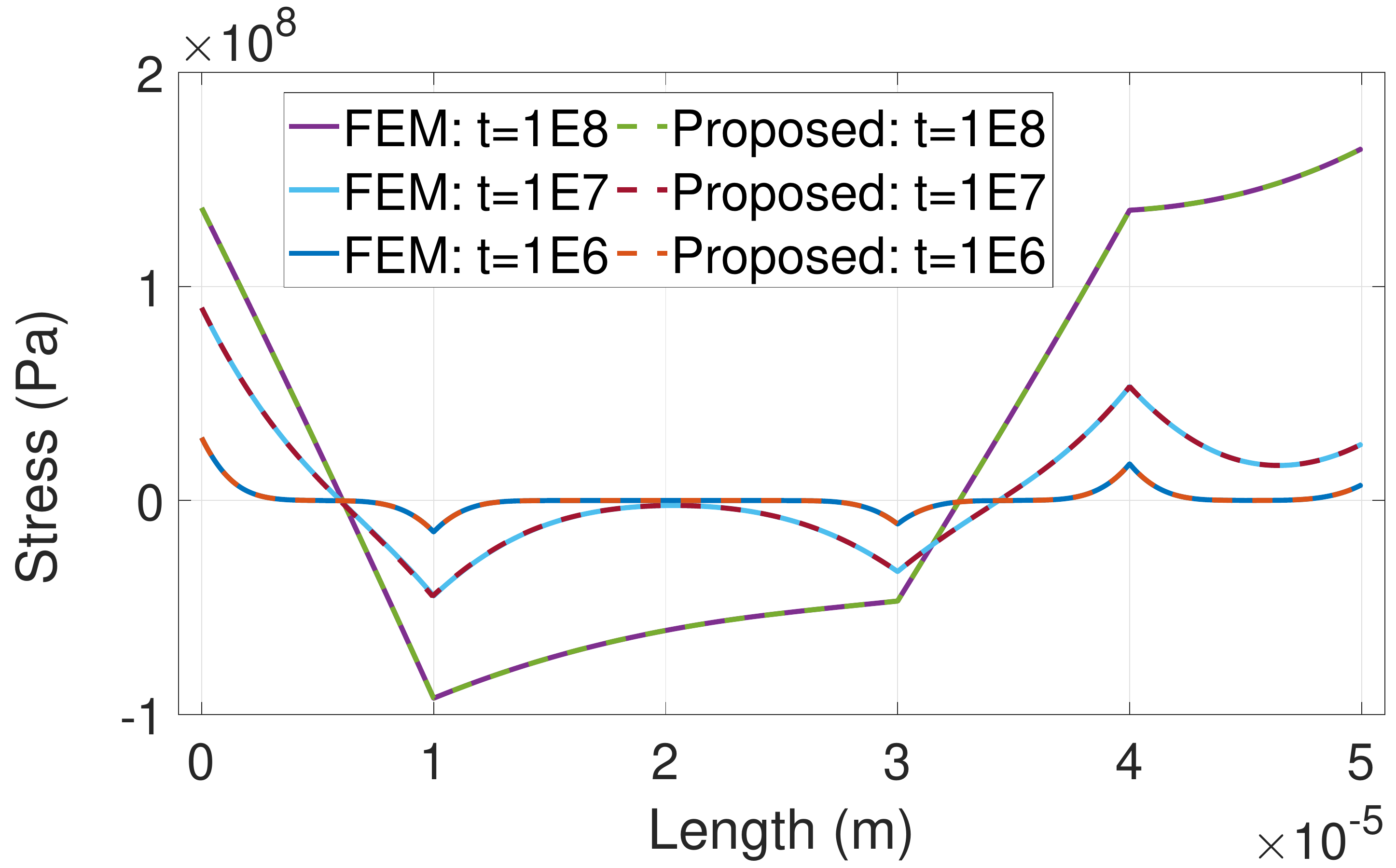}
	\label{fig:1d1_1221}}
		\subfigure[]{
		\includegraphics[width=0.46\columnwidth]{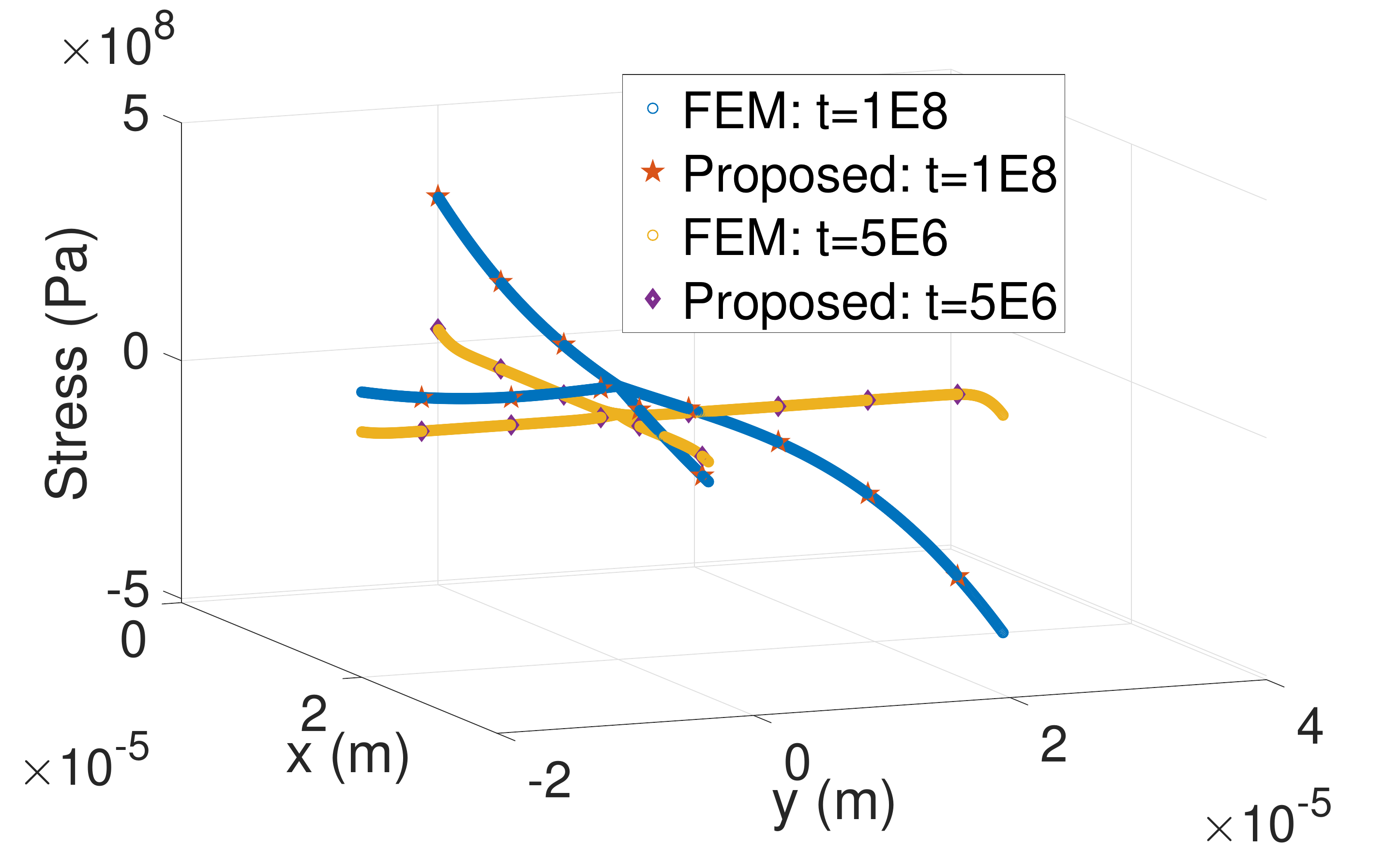}
	\label{fig:2d1_1111}}
	\subfigure[]{
		\includegraphics[width=0.46\columnwidth]{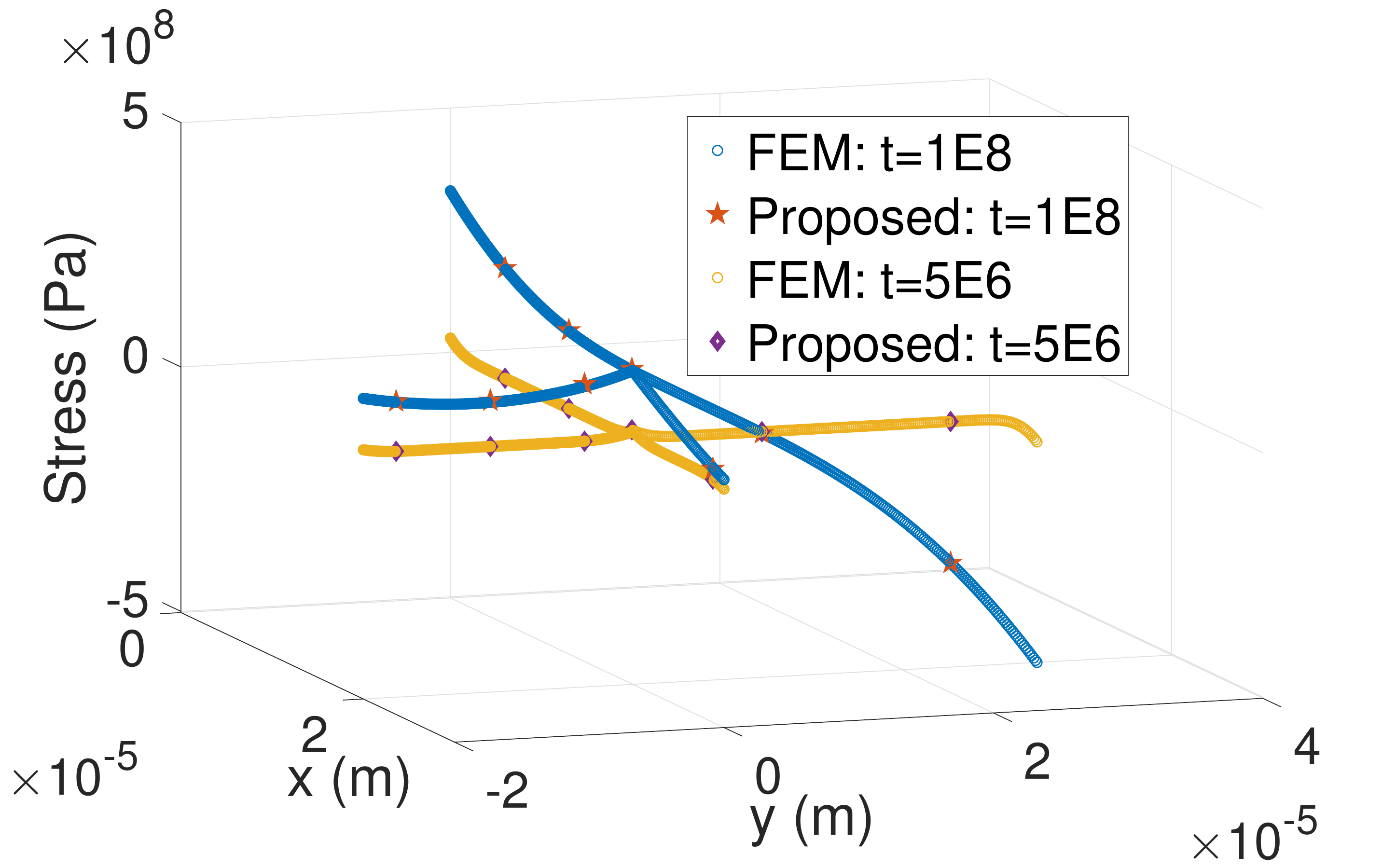}
	\label{fig:2d1_1221}}
	\caption{Stress comparisons of a four-segment straight wire and a cross-shaped five-terminal wire between the proposed method and FEM under different widths: (a) and (c): $w_1=w_2=w_3=w_4=0.1\ \mu m$; (b) and (d): $w_1=w_4=0.1\ \mu m,w_2=w_3=0.2\ \mu m$.}
\label{fig:2d1}
\end{figure}
\begin{figure}[t]
	\centering
	\includegraphics[width=0.85\columnwidth]{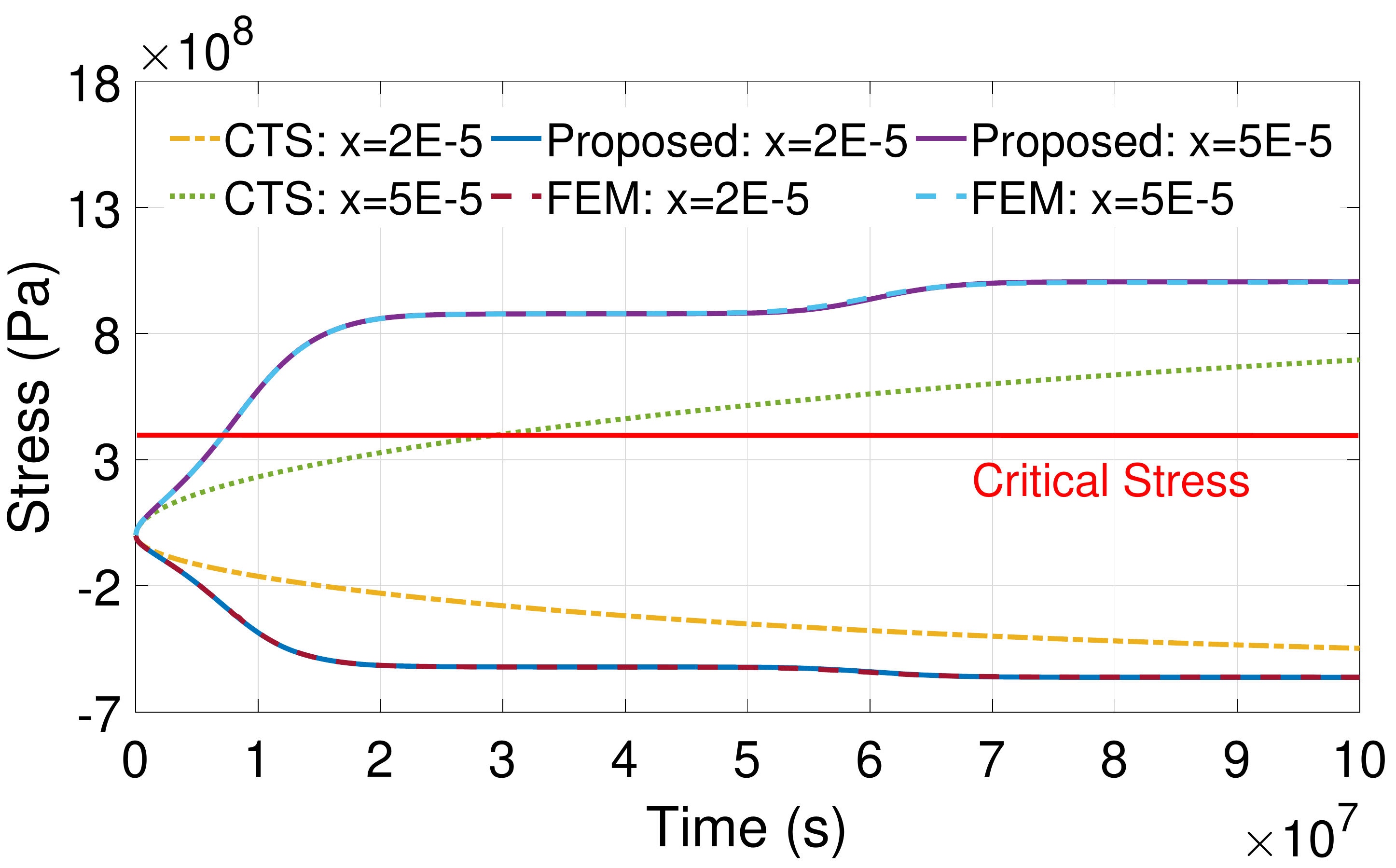}
	\caption{Time-varying temperature dependent stress comparison of a two-segment straight wire versus time. The wire is set as $L_1=20\ \mu m,\ L_2=30\ \mu m,j_1=4\times 10^9\ A/m^2,\ j_2=-1\times 10^{10}\ A/m^2$. CTS represents stress evolution under constant temperature ($350\ K$).}
	\label{fig:1d2dy}
\end{figure}

For the dynamic EM model under time-varying temperature, Fig.~\ref{fig:1d2dy} shows stress evolution at the middle node and the right terminal versus time in the range from $0$s to $1\times 10^8$s. The interconnect tree is configured as a two-segment wire of $L_1=2\ \mu m,\ L_2=3\ \mu m,\ j_1=4\times 10^9\ A/m^2,\ j_2=-1\times 10^{10}\ A/m^2$. We employ a 1-layer MLP with 100 neurons to perform the nonlinear transformation of temporal variables. The time-varying temperature profile and the stress evolution under constant temperature (CTS) are plotted in Fig.~\ref{fig:1d2dy}. It demonstrates that the dynamic model can achieve stress evolution along the whole interconnect tree with 0.62\% relative error against FEM. Although the time-varying temperature has the same average value as the constant temperature, it shows faster evolution speed and shorter void nucleation time under time-varying temperature.

\subsection{Performance Analysis}\label{sec:perf}
\begin{table}[t]
	\caption{Relative error under different numbers of discrete series of Gauss-Legendre integration ($N_g$), and different numbers of training data ($N_c$).}
	\centering
	\setlength{\tabcolsep}{3.8mm}{
		\begin{tabular}{|c|c|c|c|}
			\hline
			\diagbox {$N_c$}{$N_g$}& 8 & 16 & 32 \\ \hline
			10 & 5.31e-3 & 1.39e-3 & 3.11e-3 \\ \hline
			20 & 1.42e-3 & 1.27e-3 & 9.39e-4 \\ \hline
			30 & 6.46e-4 & 5.30e-4 & 5.08e-4 \\ \hline
	\end{tabular}}
	\label{tab:number}
\end{table}
In Table~\ref{tab:number} we report the resulting relative error under different numbers of integral discrete series and training data, while keeping the 5-layer MLP with 50 neurons per layer fixed. The relative error describes the error of stress evolution along the whole interconnect tree at 10 specified time points from $1\times 10^5$s to $1\times 10^8$s compared with FEM. It can be observed that when the number of integral series is larger than 8,  the relative error is reduced as the number of training data increases. It is shown the general trend that prediction accuracy is increased as the number $N_g$ is increased. However, it will cost more computational time in the integral operation with more discrete series.
Considering the trade-off between accuracy and calculation speed, we set the number of integral series as $N_g=8$ for fast trial function calculation. 
Table~\ref{tab:DM} shows the systematic studies of the MLP structure configuration with different numbers of layers and neurons per layer, while the number of training data and integral series are set as $N_c=30,\ N_g=8$. As expected, it can be observed from Table~\ref{tab:DM} that the prediction accuracy is increased as the number of layers and neurons is increased. In this work, we choose a 5-layer MLP with 50 neurons per layer to construct the trial function \eqref{eq:final}.

\begin{table}[t]
	\caption{Relative error under different numbers of hidden layers and neurons per layer in the MLP.}
	\centering
	\setlength{\tabcolsep}{3.8mm}{
		\begin{tabular}{|c|c|c|c|}
			\hline
			\diagbox {Layers}{Neurons} & 30 & 40 & 50 \\ \hline
			3 & 8.78e-4 & 6.31e-4 & 8.05e-4 \\ \hline
			4 & 8.70e-4 & 8.84e-4 & 6.57e-4 \\ \hline
			5 & 7.19e-4 & 9.55e-4 & 6.09e-4 \\ \hline
	\end{tabular}}
	\label{tab:DM}
\end{table}

\begin{figure}[t]
	\centering 
	\subfigure[]{
		\includegraphics[width=0.46\columnwidth]{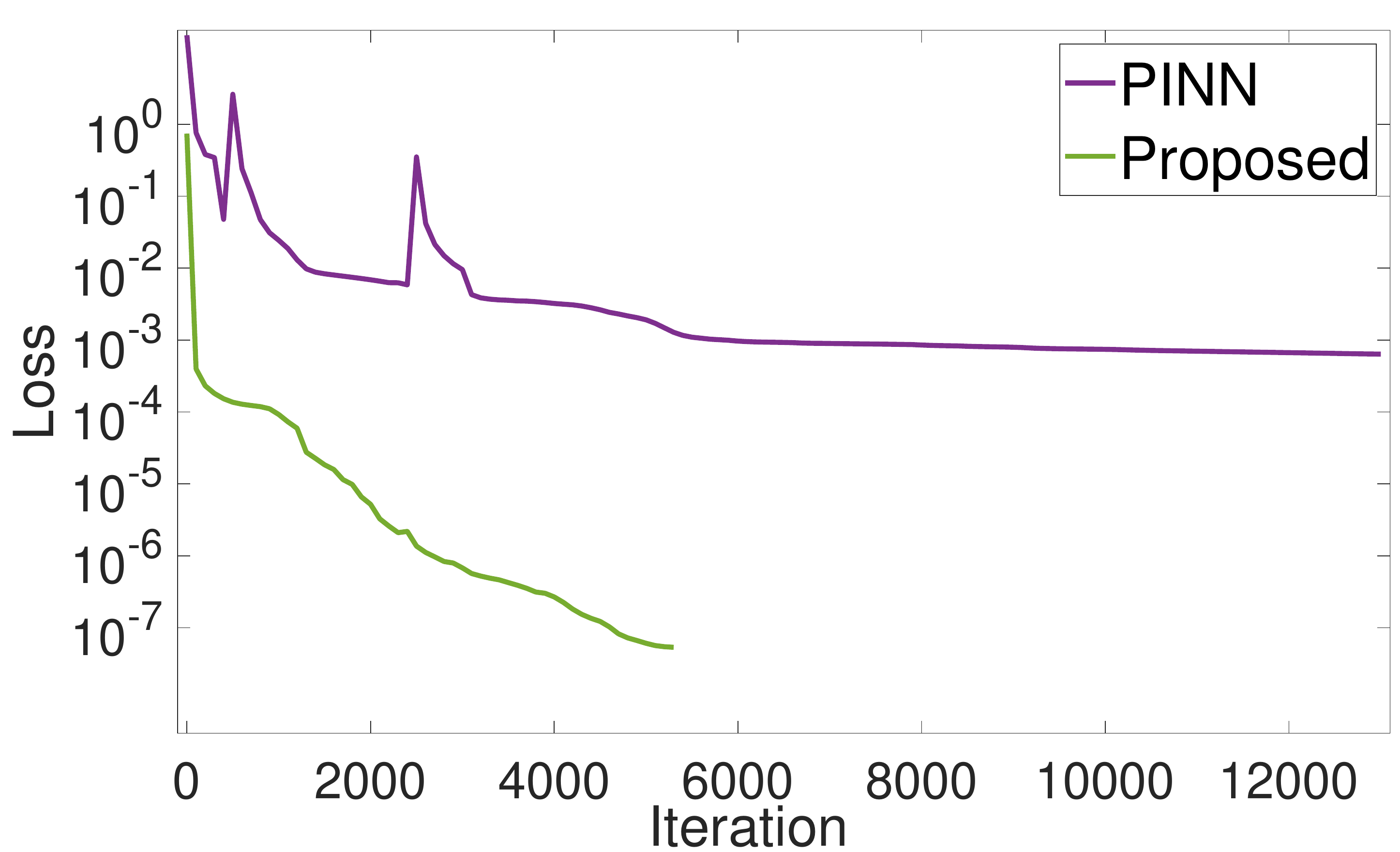}
		\label{fig:lossite}}
	\subfigure[]{
		\includegraphics[width=0.46\columnwidth]{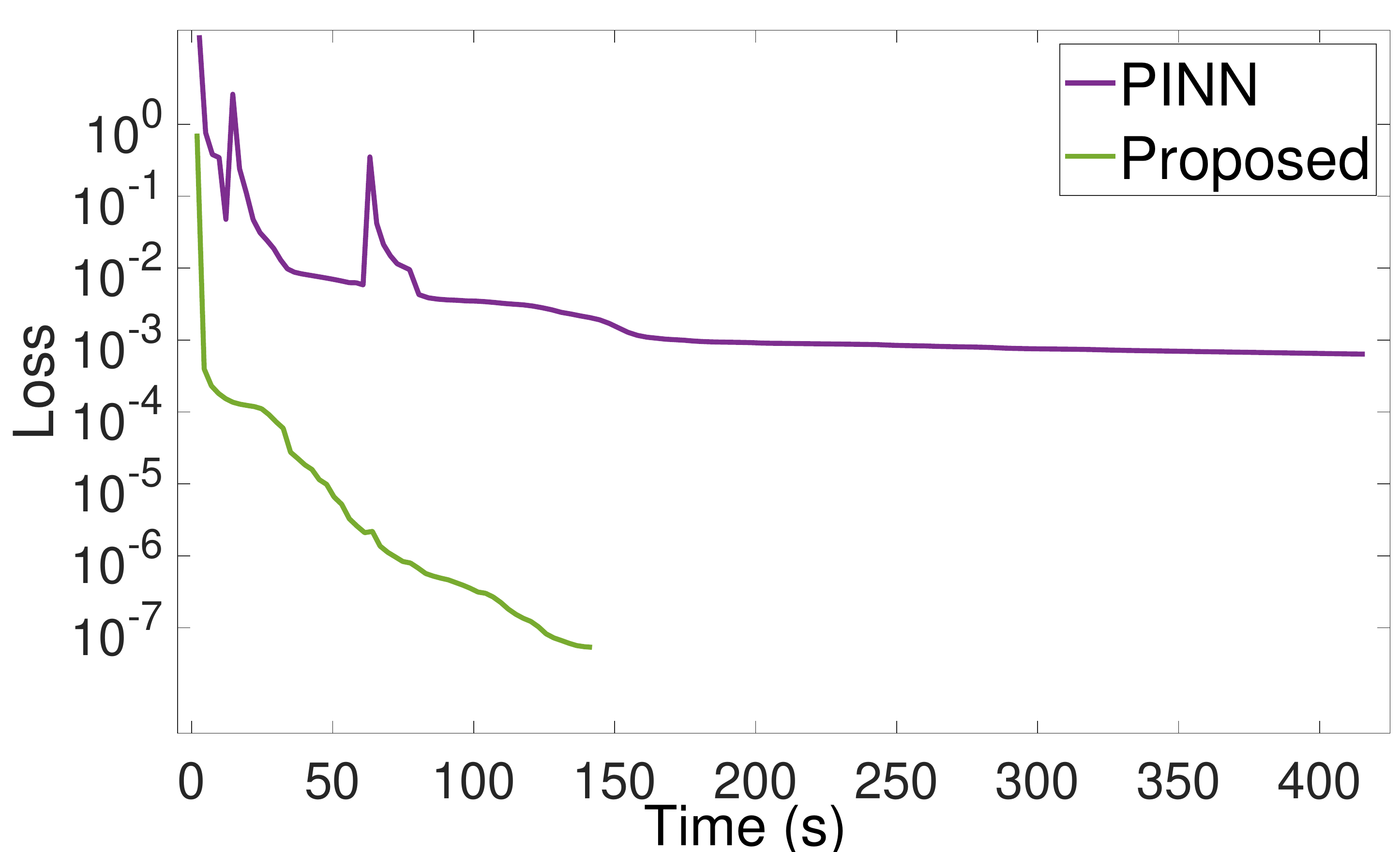}
		\label{fig:lossstim}}
	\caption{Loss vs (a) iteration and (b) time by PINN and the proposed method. Training of the proposed method is early stopped after convergence.}
	\label{fig:loss}
\end{figure}

\begin{figure}[tb]
	\centering
	\includegraphics[width=0.85\columnwidth]{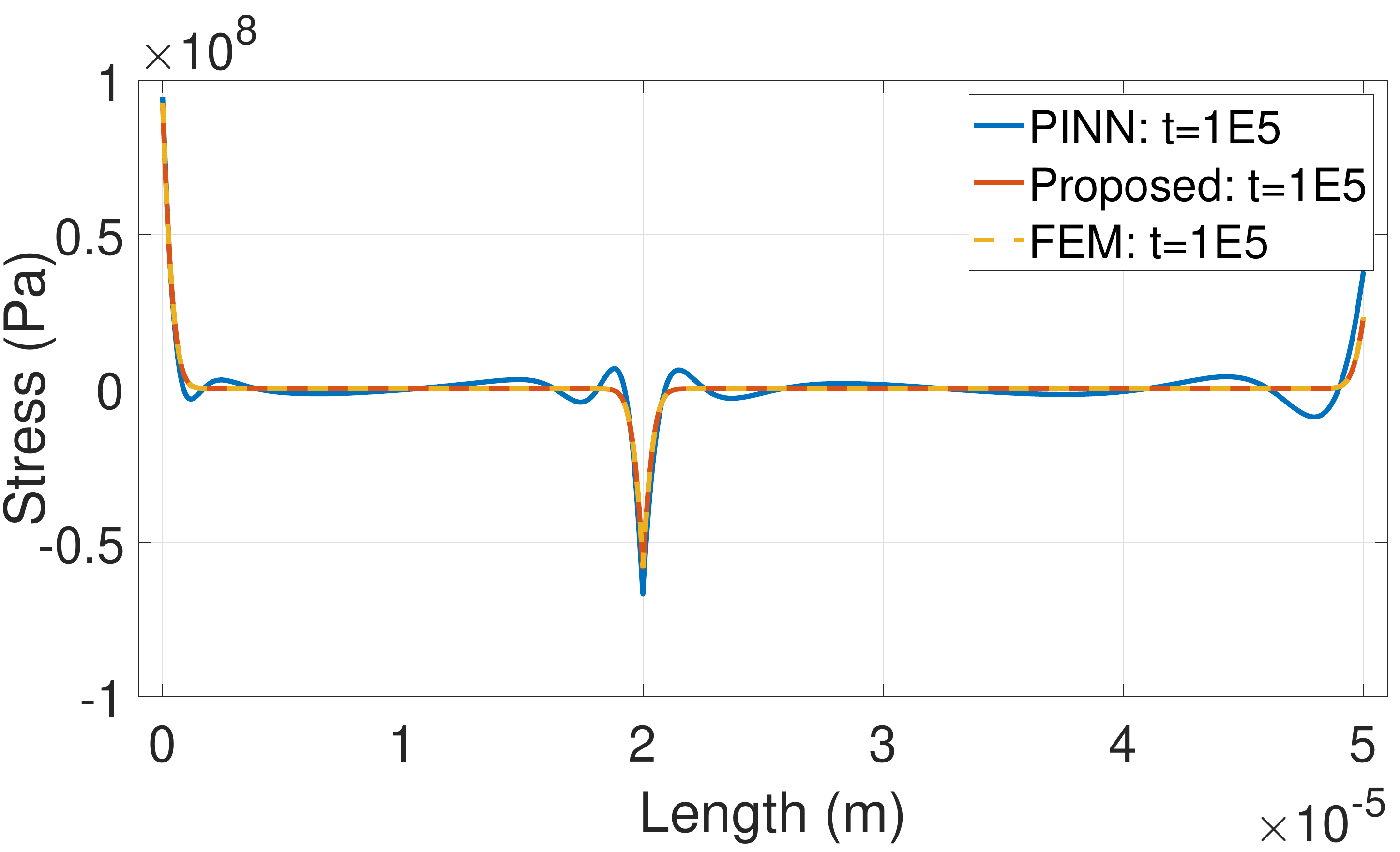}
	\caption{{\color{black}Stress comparison of the two-segment interconnect wires at $t=1\times 10^{5}$s.}}
	\label{fig:t1e5}
\end{figure}

Then, we compare the proposed method with a state-of-the-art learning-based method, the physics-informed neural network (PINN), on the change process of the objective function. PINN is proposed for tasks respecting any given laws of physics governed by PDEs through machine learning \cite{PINN2019:Journal}. We employ a 10-layer MLP with 50 neurons per layer in PINN
and a 5-layer MLP with 50 neurons per layer in the proposed method. The training data size and number of integral series are fixed to $N_c=30,N_g=8$. Figs.~\ref{fig:lossite} \& \ref{fig:lossstim} show the value change of objective function versus the iteration step and the training time when obtaining stress evolution of a four-segment straight wire through PINN and the proposed method. We plot these figures together since the objective functions of these methods are approximately the same.
It can be observed that the loss of the proposed method declines to a lower value with less training time and iteration steps than PINN. The proposed method is early stopped within 150 seconds, while PINN cannot converge to a satisfactory value after 400 seconds. {\color{black}After 13k training iterations, the proposed method and PINN show relative errors of $0.07\%,\ 3.98\%$, respectively.} Hence, compared with PINN, the proposed method is capable of achieving more satisfactory accuracy with less training time. {\color{black}For further runtime comparison, the runtime of COMSOL is $6.49$s, and it costs $4.78$s, $420$s for the training of the proposed method and the PINN-based method to obtain stress evolution with $0.94\%,\ 3.98\%$ relative errors compared with COMSOL, respectively.}
\begin{figure}[tb]
	\centering 
	\subfigure[]{
		\includegraphics[width=0.85\columnwidth]{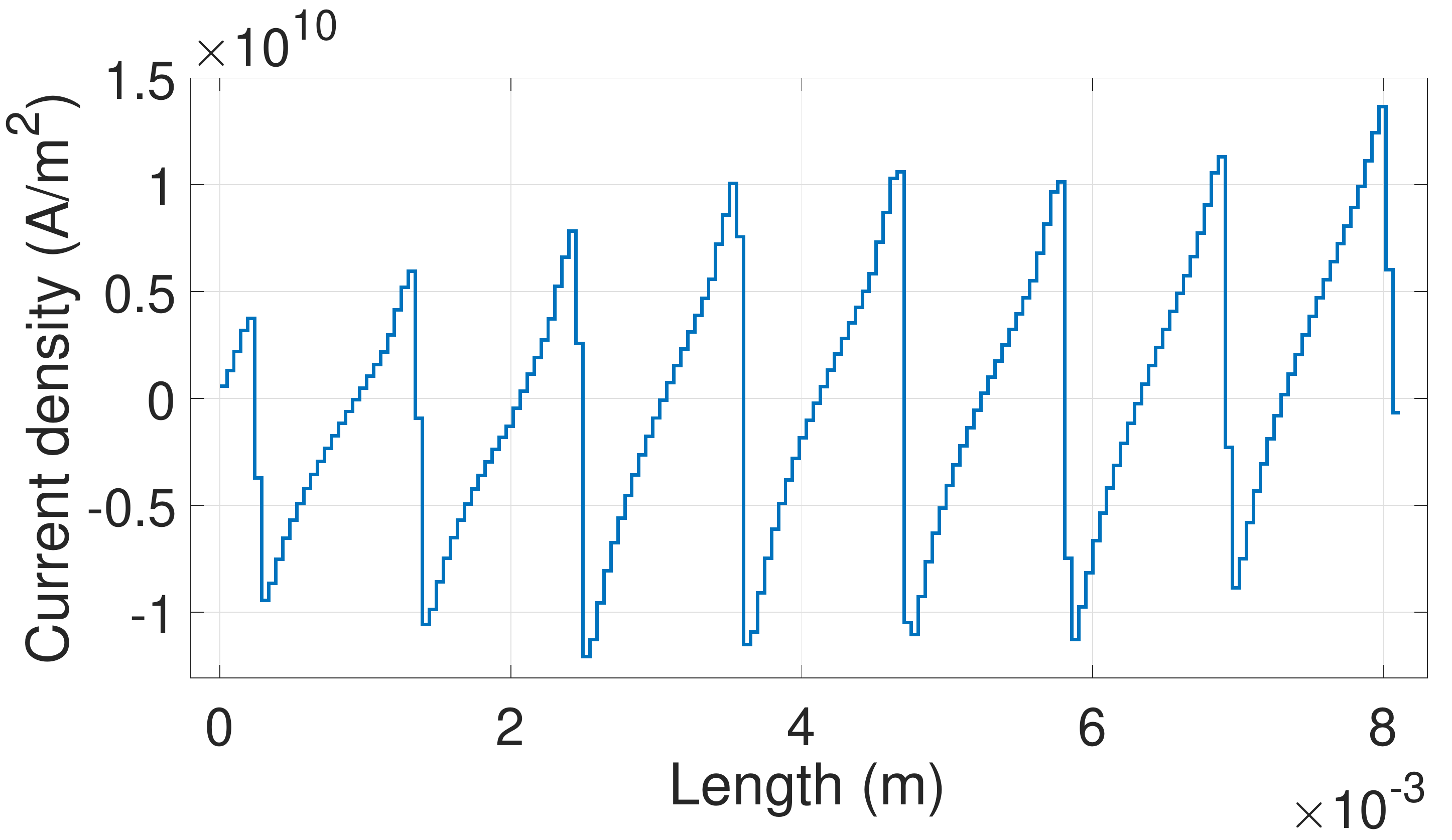}
		\label{fig:1d168j}}
	\subfigure[]{
		\includegraphics[width=0.85\columnwidth]{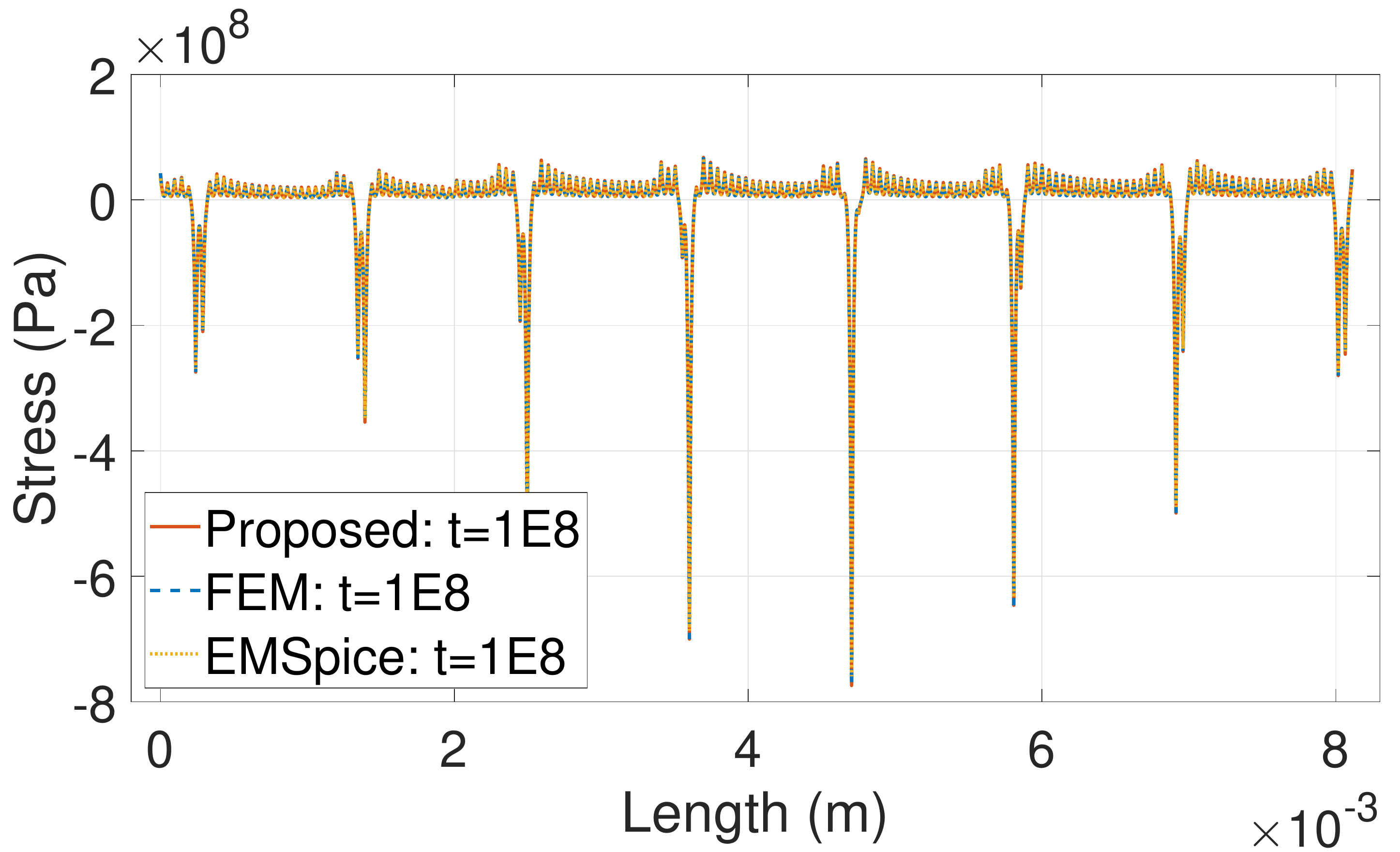}
		\label{fig:1d168}}
	\caption{(a) Current density configuration of a 168-segment interconnect tree; (b) Stress comparison of the 168-segment interconnect tree between the proposed method, FEM and EMSpice.}
	\label{fig:168}
\end{figure}

{\color{black}
Moreover, due to the wide temporal range in the stress evolution evaluation, the performance of the learned model PINN on training data will not be easily maintained on validation data.
As shown in Fig.~\ref{fig:t1e5}, after neural network training of datasets with the time range $0\sim 1\times 10^8$s, PINN shows inaccurate prediction at the early time $t=1\times 10^5$s, while our proposed method shows good agreement with the results obtained by FEM due to the employment of trial function satisfying the physical constraints.}

\subsection{Scalability Analysis}

In this section, we validate the accuracy of the proposed method in multi-segment interconnect trees by FEM and compare the performance with EMSpice simulator. 
In the simulation, we employ a 5-layer MLP with 50 neurons per layer and set the number of integral series as $N_g=8$. The number of training data is set in the range $N_c=10\sim 30$. In the training phase, if a sufficient number of temporal collocation points $t_{i,k}$ in the objective function \eqref{eq:loss} are generated for the network training, we can reduce $N_c$ when the number of segments is increased.
\subsubsection{Multi-segment straight interconnect tree}

\begin{table}[h]
	\centering
	\caption{{\color{black}Scalability performance comparison between EMSpice, COMSOL and the proposed method on increasing $n$-segmented straight interconnect trees.}}
	\setlength{\tabcolsep}{0.8mm}{
\begin{tabular}{|c|c|c|cccccc|}
\hline
\multirow{3}{*}{$n$-segment} & EMSpice & {\color{black}COMSOL} & \multicolumn{6}{c|}{Proposed} \\ \cline{2-9} 
 & \multirow{2}{*}{\begin{tabular}[c]{@{}c@{}}$t_{ems}$\\ (s)\end{tabular}} & \multirow{2}{*}{\begin{tabular}[c]{@{}c@{}}\color{black}$t_{com}$\\ \color{black}(s)\end{tabular}} & \multicolumn{2}{c|}{Inference} & \multicolumn{1}{c|}{Training} & \multicolumn{1}{c|}{\color{black}Total} & \multicolumn{2}{c|}{Error} \\ \cline{4-9} 
 &  &  & \multicolumn{1}{c|}{\begin{tabular}[c]{@{}c@{}}$t_{pre}$\\ (s)\end{tabular}} & \multicolumn{1}{c|}{\begin{tabular}[c]{@{}c@{}}$t_{inf}$\\ (s)\end{tabular}} & \multicolumn{1}{c|}{\begin{tabular}[c]{@{}c@{}}$t_{tra}$\\ (s)\end{tabular}} & \multicolumn{1}{c|}{\begin{tabular}[c]{@{}c@{}}\color{black}$t_{tot}$\\ \color{black}(s)\end{tabular}} & \multicolumn{1}{c|}{\begin{tabular}[c]{@{}c@{}}$\delta$\\ (\%)\end{tabular}} & \begin{tabular}[c]{@{}c@{}}$\delta_{con}$\\ (\%)\end{tabular} \\ \hline
22 & 1.11 & {\color{black}26} & \multicolumn{1}{c|}{0.02} & \multicolumn{1}{c|}{0.07} & \multicolumn{1}{c|}{15.87} & \multicolumn{1}{c|}{\color{black}15.96} & \multicolumn{1}{c|}{0.05} & 0.02 \\ \hline
58 & 7.91 & {\color{black}49} & \multicolumn{1}{c|}{0.03} & \multicolumn{1}{c|}{0.15} & \multicolumn{1}{c|}{21.26} & \multicolumn{1}{c|}{\color{black}21.44} & \multicolumn{1}{c|}{0.02} & 0.02 \\ \hline
76 & 12.45 & {\color{black}80} & \multicolumn{1}{c|}{0.04} & \multicolumn{1}{c|}{0.16} & \multicolumn{1}{c|}{19.90} & \multicolumn{1}{c|}{\color{black}20.10} & \multicolumn{1}{c|}{0.05} & 0.05 \\ \hline
109 & 25.34 & {\color{black}102} & \multicolumn{1}{c|}{0.08} & \multicolumn{1}{c|}{0.26} & \multicolumn{1}{c|}{26.72} & \multicolumn{1}{c|}{\color{black}27.06} & \multicolumn{1}{c|}{0.03} & 0.03 \\ \hline
168 & 48.72 & {\color{black}110} & \multicolumn{1}{c|}{0.10} & \multicolumn{1}{c|}{0.23} & \multicolumn{1}{c|}{56.65} & \multicolumn{1}{c|}{\color{black}56.98} & \multicolumn{1}{c|}{0.09} & 0.04 \\ \hline
236 & 120.23 & {\color{black}115} & \multicolumn{1}{c|}{0.11} & \multicolumn{1}{c|}{0.24} & \multicolumn{1}{c|}{32.55} & \multicolumn{1}{c|}{\color{black}32.90} & \multicolumn{1}{c|}{0.04} & 0.04 \\ \hline
367 & 583.17 & {\color{black}206} & \multicolumn{1}{c|}{0.19} & \multicolumn{1}{c|}{0.28} & \multicolumn{1}{c|}{173.18} & \multicolumn{1}{c|}{\color{black}173.65} & \multicolumn{1}{c|}{1.79} & 0.18 \\ \hline
439 & 1353.65 & {\color{black}277} & \multicolumn{1}{c|}{0.23} & \multicolumn{1}{c|}{0.36} & \multicolumn{1}{c|}{196.93} & \multicolumn{1}{c|}{\color{black}197.52} & \multicolumn{1}{c|}{1.14} & 0.14 \\ \hline
571 & 3569.25 & {\color{black}398} & \multicolumn{1}{c|}{0.36} & \multicolumn{1}{c|}{0.43} & \multicolumn{1}{c|}{241.93} & \multicolumn{1}{c|}{\color{black}242.72} & \multicolumn{1}{c|}{0.82} & 0.13 \\ \hline
{\color{black}702} & {\color{black}4046.67} & {\color{black}556} &\multicolumn{1}{c|}{\color{black}0.41} & \multicolumn{1}{c|}{\color{black}0.56} & \multicolumn{1}{c|}{\color{black}375.45} & \multicolumn{1}{c|}{\color{black}376.42} & \multicolumn{1}{c|}{\color{black}2.97} & {\color{black}0.52} \\ \hline
{\color{black}801} &{\color{black}8071.43} & {\color{black}715} &\multicolumn{1}{c|}{\color{black}0.49} & \multicolumn{1}{c|}{\color{black}0.64} & \multicolumn{1}{c|}{\color{black}582.50} & \multicolumn{1}{c|}{\color{black}583.63} & \multicolumn{1}{c|}{\color{black}2.73} & {\color{black}0.20} \\ \hline
\end{tabular}}
	\label{tab:ntime}
\end{table}

To further validate the performance of the proposed method in multi-segment straight interconnect trees, we analyze the EM-induced stress on interconnects extracted from International Business Machines Corporation (IBM) power grid benchmark IBMPG2-IBMPG4 structure~\cite{Nass2008:ASPDAC}.
Figs.~\ref{fig:1d168j} \&~\ref{fig:1d168} show the current density configuration of a 168-segment interconnect tree extracted from IBMPG2 and the stress evolution comparison, which demonstrates good agreements within $0.04\%$ error. The results demonstrate that the proposed method can be implemented for stress evolution analysis on multi-segment straight interconnect tress with promising accuracy.

\begin{table*}[h]
	\caption{{\color{black}Relative errors under a simple linear regression, a single-layer neural network and a 5-layer MLP.}}
	\centering
	\setlength{\tabcolsep}{2.7mm}{
	\begin{tabular}{|c|c|c|c|c|c|c|c|c|c|c|c|}
\hline
\diagbox {Network}{$n$-segment}& 22 & 58 & 76 & 109 & 168 & 236 & 367 & 439 & 571& 702 & 801 \\ \hline
Linear Regression & 1.8e-4 & 1.4e-4 & 2.9e-4 & 3.3e-4 & 1.1e-3 & 1.9e-4 & 4.8e-2 & 5.6e-2 & 3.4e-2 & 3.2e-2 & 1.8e-1 \\ \hline
Single-layer Neural Network & 1.8e-4 & 5.3e-4 & 8.7e-4 & 1.8e-3 & 7.1e-4 & 7.3e-4 & 2.8e-2 & 2.7e-2 & 1.0e-2 & 2.5e-2 & 1.4e-1 \\ \hline
5-layer MLP & 1.8e-4 & 2.3e-4 & 5.3e-4 & 3.3e-4 & 3.9e-4 & 3.7e-4 & 1.8e-3 & 1.4e-3 & 1.3e-3 & 5.2e-3 & 2.0e-3 \\ \hline
\end{tabular}}
	\label{tab:linear}
\end{table*}

Moreover, we perform the proposed method and EMSpice simulations on increasing $n$-segmented straight interconnect trees to demonstrate the computational savings and satisfactory accuracy of the proposed method, shown in Table~\ref{tab:ntime}.
{\color{black}Here, we record the runtime of EMSpice employing 100 temporal iterations as $t_{ems}$ and the sum of runtime for COMSOL modeling and high-accuracy computation as $t_{com}$, respectively.}
For the inference phase of the proposed method, the runtime $t_{pre}$ and $t_{inf}$ are the data preparation time for the observed space-time input data and the inference time for obtaining stress development at 10 specified time points from $1\times 10^5$s to $1\times 10^8$s. The notation $t_{tra}$ represents the training time for each case within 2k iterations and $\delta$ represents the corresponding relative error of the trained model. {\color{black}The total runtime $t_{tot}$ of the proposed method is the sum of $t_{pre},\ t_{inf}$ and $t_{tra}$.}
The proposed method consumes little time for test data preparation and stress evolution inference, demonstrating great computational savings of the proposed method.
This saving is more obvious in the stress prediction of interconnect trees with more segments.
Although the adjustable parameters in the proposed method are required to be trained for specific interconnects configured with varying current densities, this can be alleviated via offline training \cite{ZHU2019:jcp}. 
{\color{black}Since the second-order optimization based method L-BFGS is employed in the proposed method, the training time scales with the grid size as $O(\beta mn)$ when the numbers of layers and neurons per layer are fixed.
The notation $m$ is a small number (typically between five and ten) related to the L-BFGS technique and $\beta$ is the number of training iterations. The accuracy of the learned model will vary with different interconnect cases.}
The proposed method will show more promising accuracy with more training iterations and we record the error of the proposed method by $\delta_{con}$ when the training is converged.
{\color{black}It can be observed from Table~\ref{tab:ntime} that when the number of interconnect segments is larger than 236, $t_{tot}$ is smaller than the execution time of EMSpice.} Besides, the proposed method shows increasing performance gain in running speed as the number of segments increases.
{\color{black}Compared with competing methods such as EMSpice and FEM, the proposed method requires less discrete integration series without a mesh generation, while keeping satisfactory approximation accuracy.

Since Table~\ref{tab:DM} shows that the error does not change much with the number of layers, we reduced MLP into a simple linear regression and a single-layer neural network, and reported the experimental results in Table~\ref{tab:linear}, showing relative errors of stress prediction on $n$-segment interconnects under three different neural networks. It can be observed that both the linear regression model and the single-layer neural network model can obtain high accuracy for interconnects with few segments. However, the accuracy of both the linear regression model and the linear regression model  will reduce as the number of segments increases, especially when the number of segments is larger than 236. By employing a 5-layer MLP, satisfactory accuracy can be achieved as the segment number increases.
}

\begin{figure}[t]
	\centering
	\includegraphics[width=0.95\columnwidth]{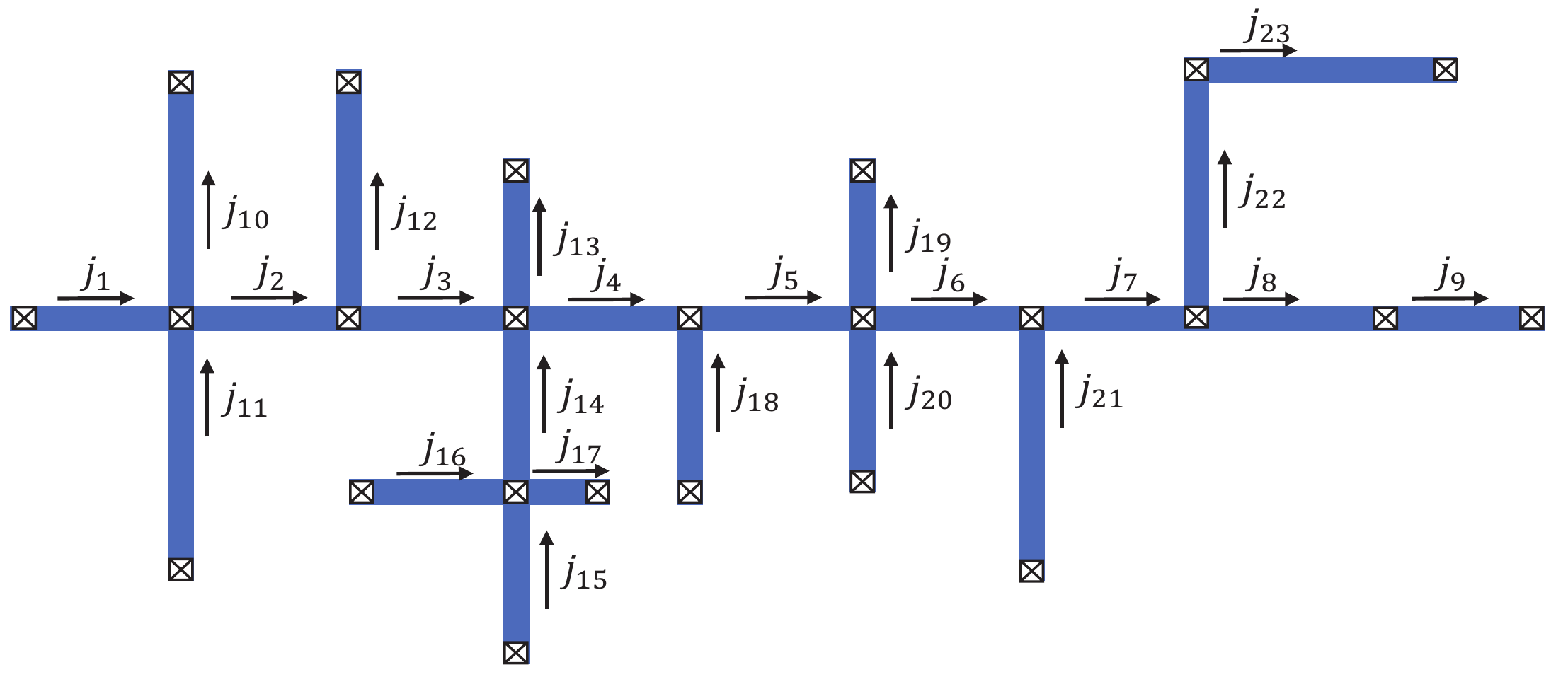}
	\caption{A 23-segment complex interconnect tree structure.}
	\label{fig:cinter}
\end{figure}
\begin{figure}[t]
	\centering
	\includegraphics[width=0.95\columnwidth]{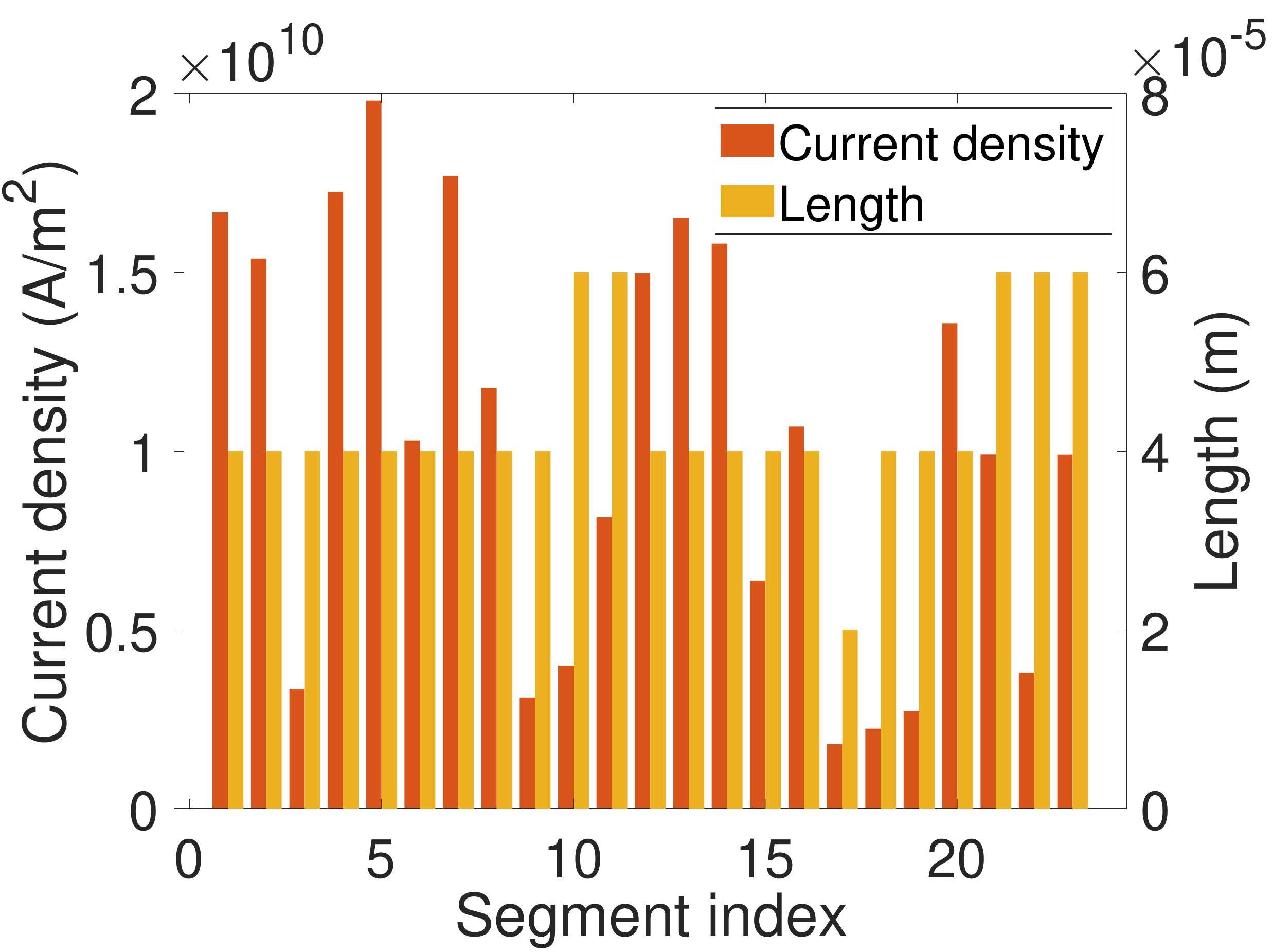}
	\caption{Configuration of current density and length for the 23-segment complex interconnect tree.}
	\label{fig:interconfig}
\end{figure}
\begin{figure}[t]
	\centering 
	\subfigure[]{
		\includegraphics[width=0.45\columnwidth]{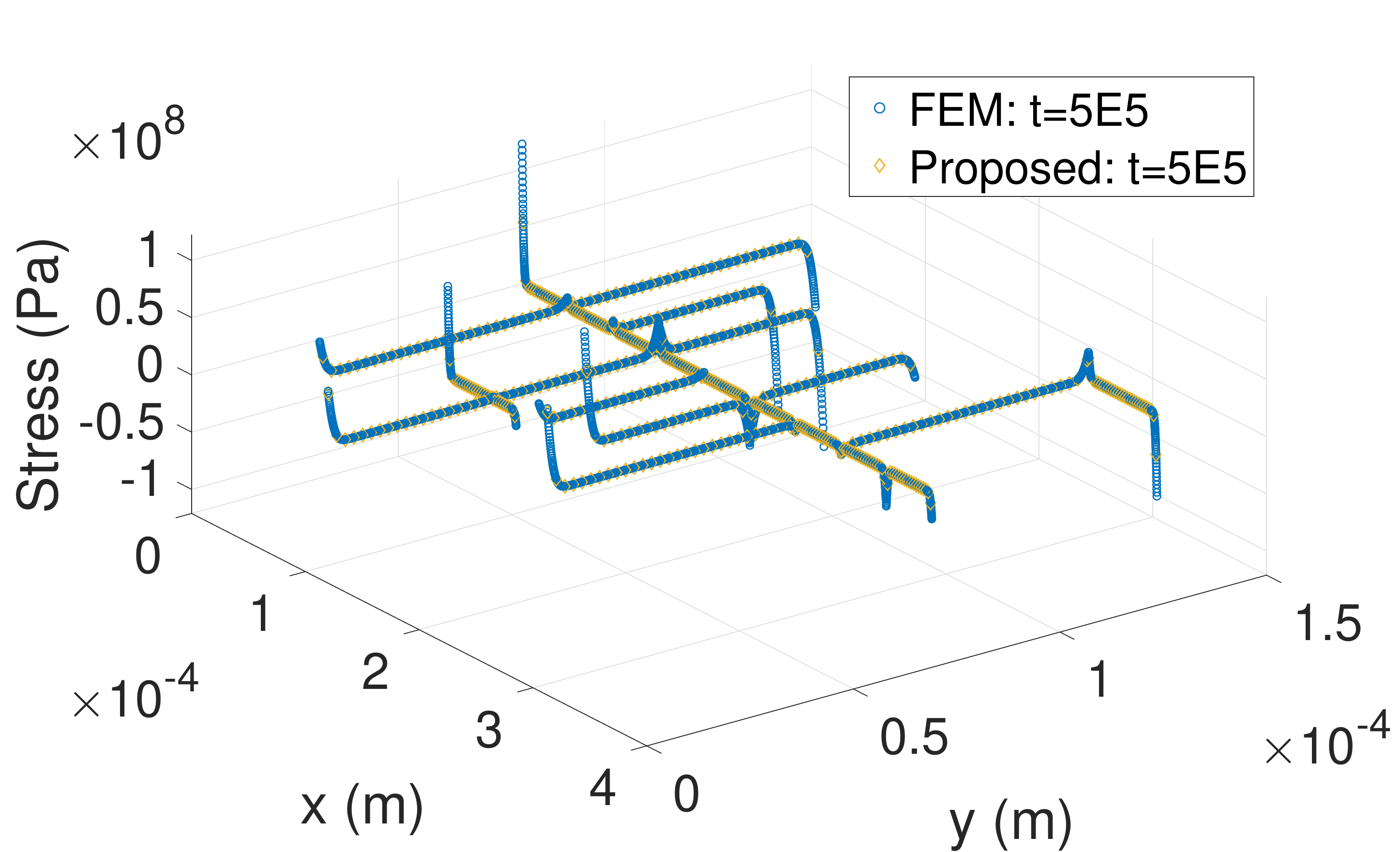}
	\label{fig:23t5e5}}
	\subfigure[]{
		\includegraphics[width=0.45\columnwidth]{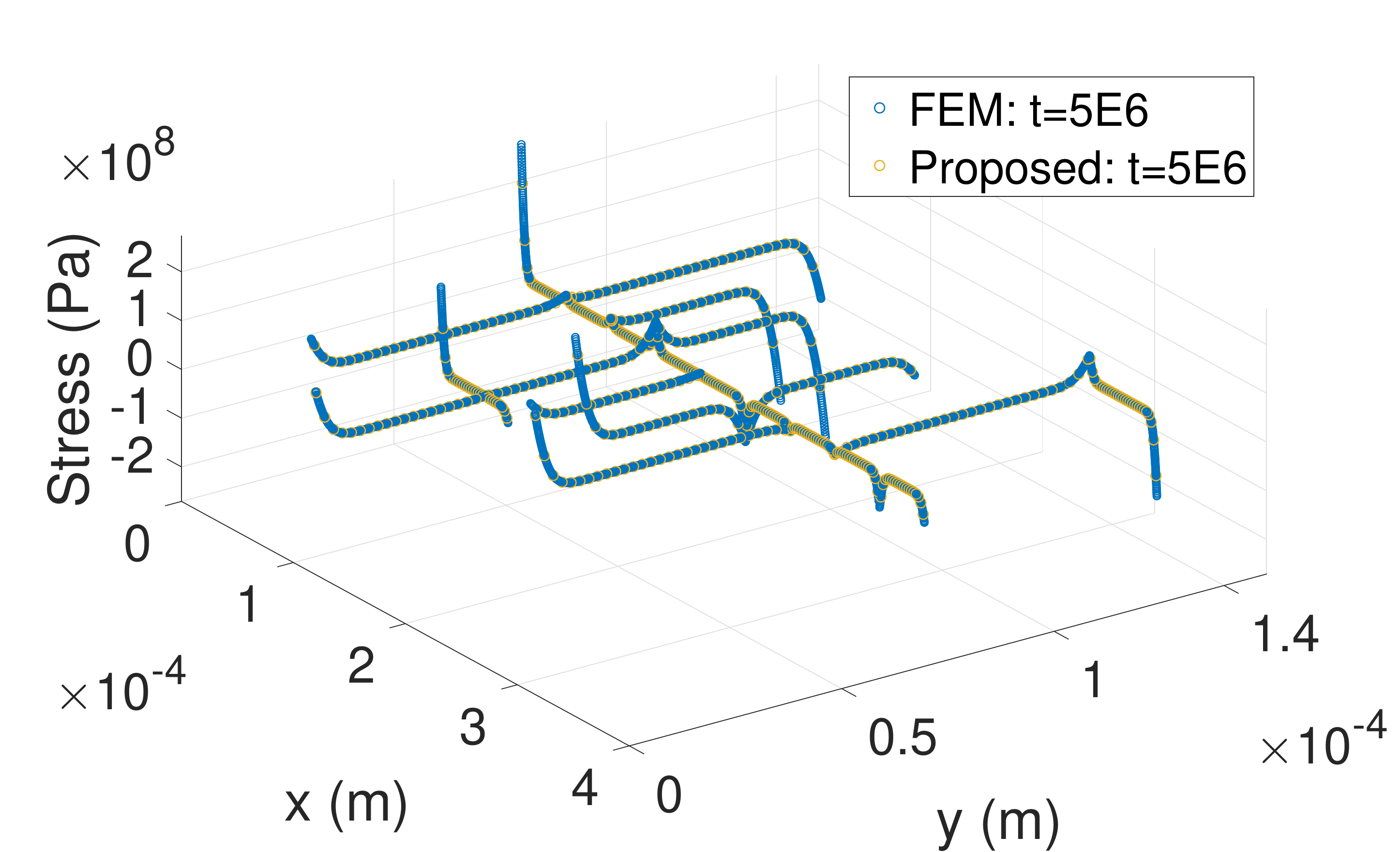}
	\label{fig:23t5e6}}
		\subfigure[]{
		\includegraphics[width=0.45\columnwidth]{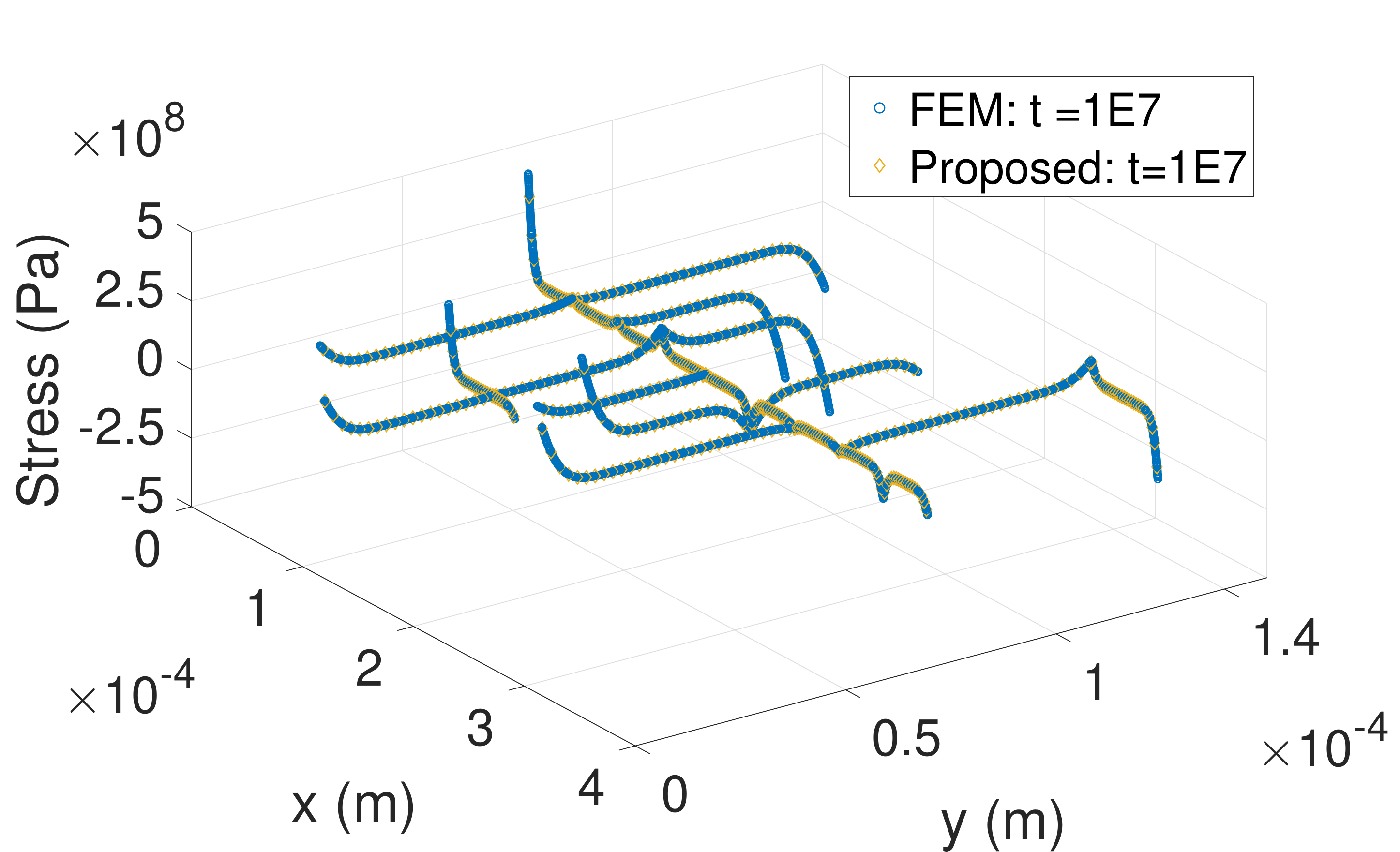}
	\label{fig:23t1e7}}
	\subfigure[]{
		\includegraphics[width=0.45\columnwidth]{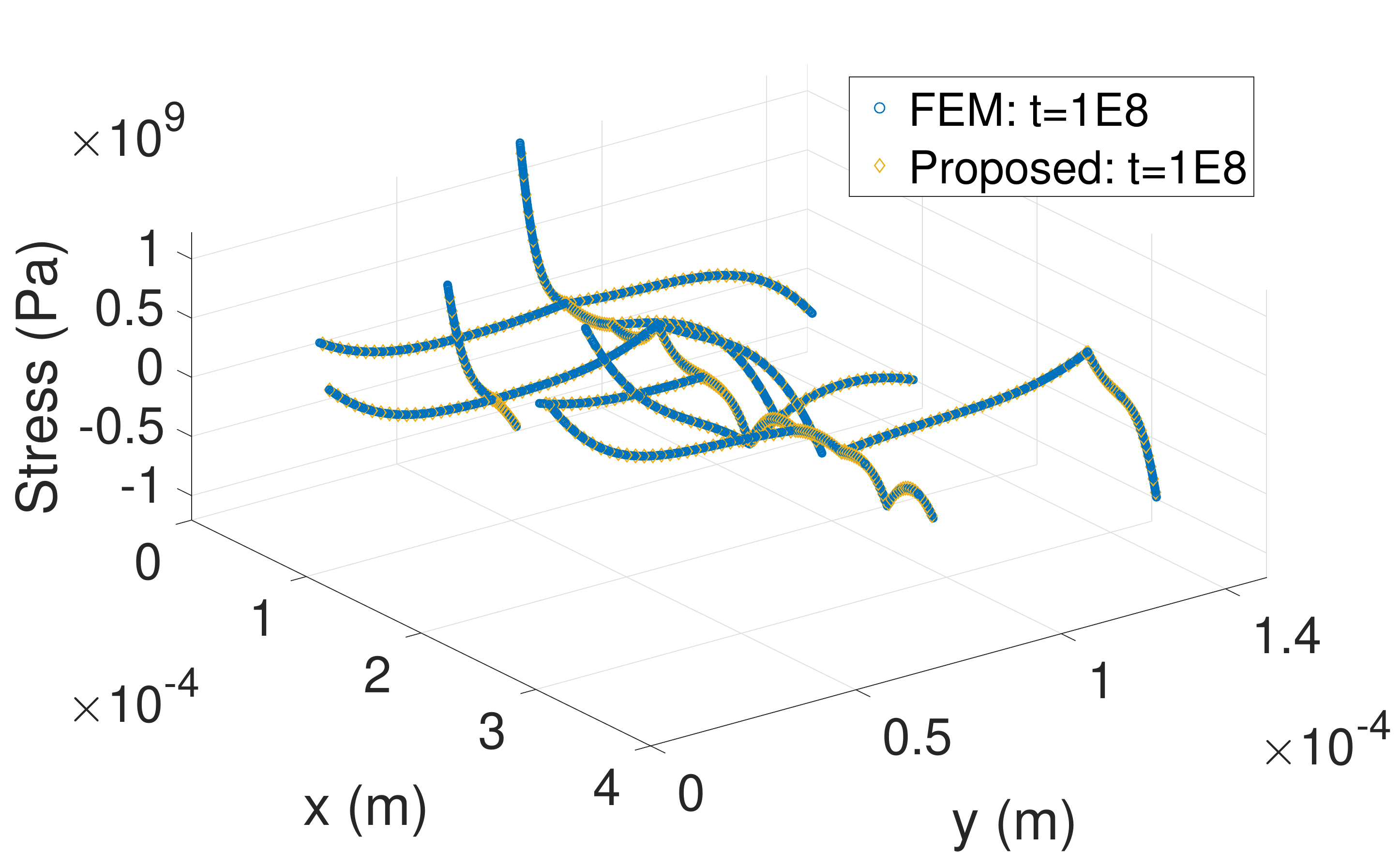}
	\label{fig:23t1e8}}
	\caption{Stress comparison of a 23-segment complex interconnect tree between the proposed method and FEM at (a) $t=5\times 10^5$s; (b) $t=5\times 10^6$s; (c) $t=1\times 10^7$s; (d) $t=1\times 10^8$s.}
	\label{fig:complex}
\end{figure}

\begin{table}[t]
	\caption{{\color{black}Scalability performance comparison between COMSOL and the proposed method on increasing $n$-segmented complex interconnect trees.}}
	\centering
	\setlength{\tabcolsep}{1.5mm}{
\begin{tabular}{|c|c|cccccc|}
\hline
\multirow{3}{*}{$n$-segment} & {\color{black}COMSOL} & \multicolumn{6}{c|}{Proposed} \\ \cline{2-8} 
 & \multirow{2}{*}{\begin{tabular}[c]{@{}c@{}}\color{black}$t_{com}$\\ \color{black}(s)\end{tabular}} & \multicolumn{2}{c|}{Inference} & \multicolumn{1}{c|}{Training} & \multicolumn{1}{c|}{\color{black}Total} & \multicolumn{2}{c|}{Error} \\ \cline{3-8} 
 &  & \multicolumn{1}{c|}{\begin{tabular}[c]{@{}c@{}}$t_{pre}$\\ (s)\end{tabular}} & \multicolumn{1}{c|}{\begin{tabular}[c]{@{}c@{}}$t_{inf}$\\ (s)\end{tabular}} & \multicolumn{1}{c|}{\begin{tabular}[c]{@{}c@{}}$t_{tra}$\\ (s)\end{tabular}} & \multicolumn{1}{c|}{\begin{tabular}[c]{@{}c@{}}\color{black}$t_{tot}$\\ \color{black}(s)\end{tabular}} & \multicolumn{1}{c|}{\begin{tabular}[c]{@{}c@{}}$\delta$\\ (\%)\end{tabular}} & \begin{tabular}[c]{@{}c@{}}$\delta_{con}$\\ (\%)\end{tabular} \\ \hline
23 & \color{black}56 & \multicolumn{1}{c|}{0.35} & \multicolumn{1}{c|}{0.15} & \multicolumn{1}{c|}{57.29} & \multicolumn{1}{c|}{\color{black}57.79} & \multicolumn{1}{c|}{1.41} & 0.38 \\ \hline
38 & \color{black}98 & \multicolumn{1}{c|}{0.50} & \multicolumn{1}{c|}{0.18} & \multicolumn{1}{c|}{65.85} & \multicolumn{1}{c|}{\color{black}66.53} & \multicolumn{1}{c|}{1.49} & 0.68 \\ \hline
84 & \color{black}235 & \multicolumn{1}{c|}{1.13} & \multicolumn{1}{c|}{0.30} & \multicolumn{1}{c|}{88.93} & \multicolumn{1}{c|}{\color{black}90.36} & \multicolumn{1}{c|}{1.47} & 0.48 \\ \hline
161 & \color{black}434 & \multicolumn{1}{c|}{2.56} & \multicolumn{1}{c|}{0.45} & \multicolumn{1}{c|}{237.24} & \multicolumn{1}{c|}{\color{black}240.25} & \multicolumn{1}{c|}{2.06} & 0.50 \\ \hline
\end{tabular}}
	\label{tab:ncomplex}
\end{table} 

\subsubsection{Multi-segment complex interconnect tree}

In real power interconnects of the standard cell, there are complex interconnect trees containing nodes with more than two adjacent segments \cite{chen2020:tcad}. 
Fig.~\ref{fig:cinter} shows the structure of a 23-segment complex interconnect tree and Fig.~\ref{fig:interconfig} shows the configured current density and length of each segment. The comparison of stress evolution solution under the constant temperature between the proposed method and FEM is shown in Fig.~\ref{fig:complex}. Furthermore, Table~\ref{tab:ncomplex} describes the scalability of increasing $n$-segmented complex interconnect trees through the proposed method. 
The results show that the proposed method consumes more time in the data preparation and the inference procedure of multi-segment complex interconnects than those of multi-segment straight interconnects due to the increasing data related to adjacent segments of the complex structure. For the cases shown in Table~\ref{tab:ncomplex}, the proposed method can achieve stress evolution with relative errors less than $2.06\%$ within time consumption $237.24$s for training and $3.01$s for testing, demonstrating more computational savings than COMSOL. The convergence errors of complex interconnect trees are less than $0.68\%$.


{\color{black}\subsection{Parameterized Label-free Modeling}
In this section, we extend our proposed method for parameterized simulations involving multiple varying geometric and current density parameters, which results in that the learned model can generalize to unseen cases. Since the stress evolution is related to the global characteristics of the interconnect wire such as geometry and current densities, the adjacent node coordinates of $\mathbb{C}_{i_j}$ are employed as the additional inputs of the MLP model. In this way, for the stress analysis of two-segment interconnect wires, the global interconnect geometry and current density are included in the input of the MLP model. We randomly generated 1k sets of two-segment interconnect wires with varying current densities and lengths shorter than $100 \mu m$ to construct training datasets. The proposed method is label-free since no prior knowledge of stress evolution (label) is required during the training procedure. We employed a 5-layer MLP with 50 neurons per layer and set $N_g=8,\ N_c=30$ for each training batch. To validate the accuracy of the learned model, Fig.~\ref{fig:parapic} shows the stress evolution results of the test cases describing two-segment interconnect wires, where the four test cases are completely unseen during training.
The proposed label-free method can extrapolate to brand new test cases with no need for retraining.
Compared with the FEM based tool COMSOL, the results of the proposed label-free method demonstrate $2.87\%$ average relative error on the four new test cases. Since no training is required for each new case, the time cost of the proposed method is $0.006$s ($1099\times$ faster than COMSOL and $13\times$ faster than EMSpice), showing significant computational savings.

\begin{figure}[t]
	\centering 
	\subfigure[]{
		\includegraphics[width=0.45\columnwidth]{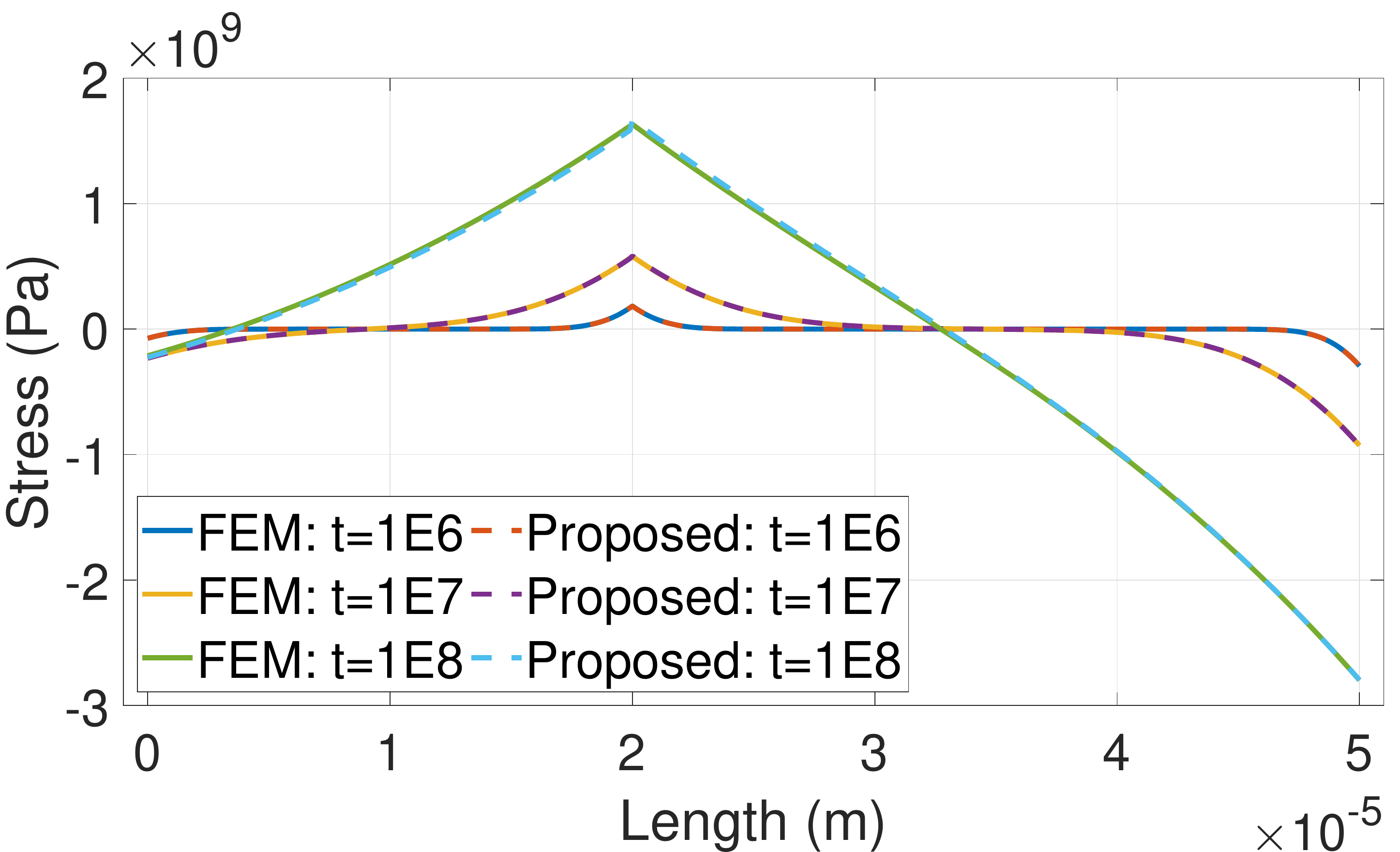}
	\label{fig:pic1}}
	\subfigure[]{
		\includegraphics[width=0.45\columnwidth]{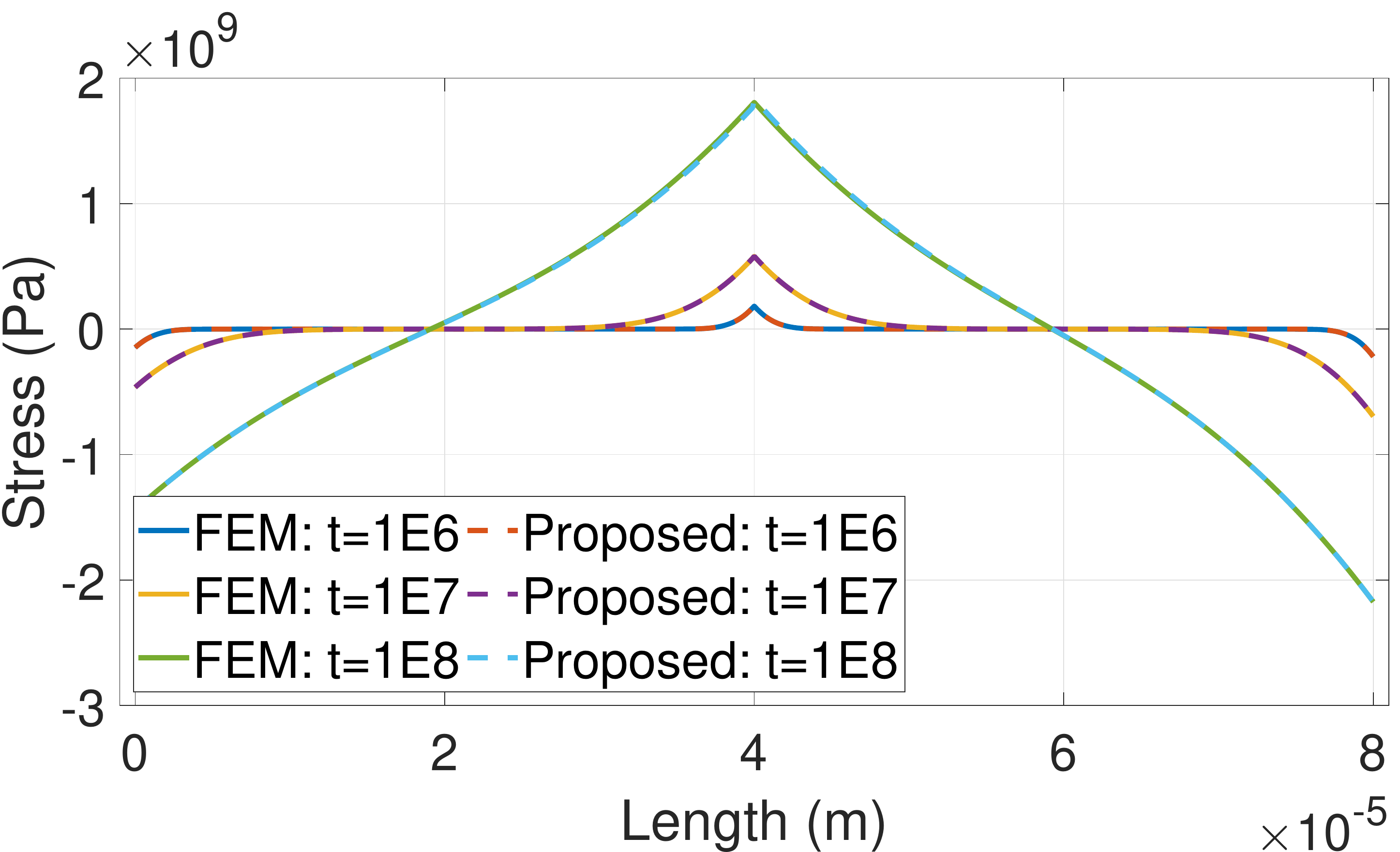}
	\label{fig:pic2}}
		\subfigure[]{
		\includegraphics[width=0.46\columnwidth]{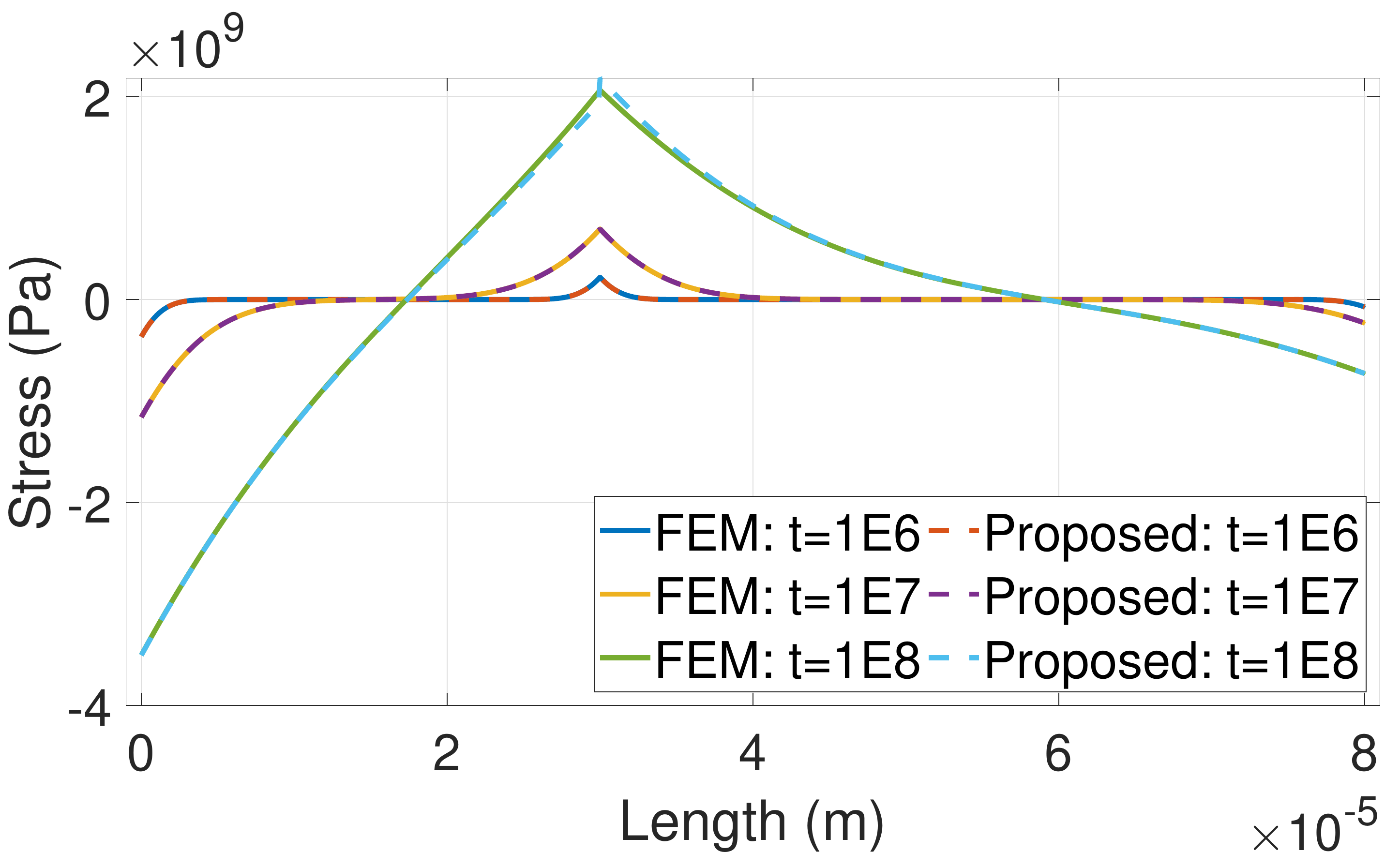}
	\label{fig:pic3}}
	\subfigure[]{
		\includegraphics[width=0.46\columnwidth]{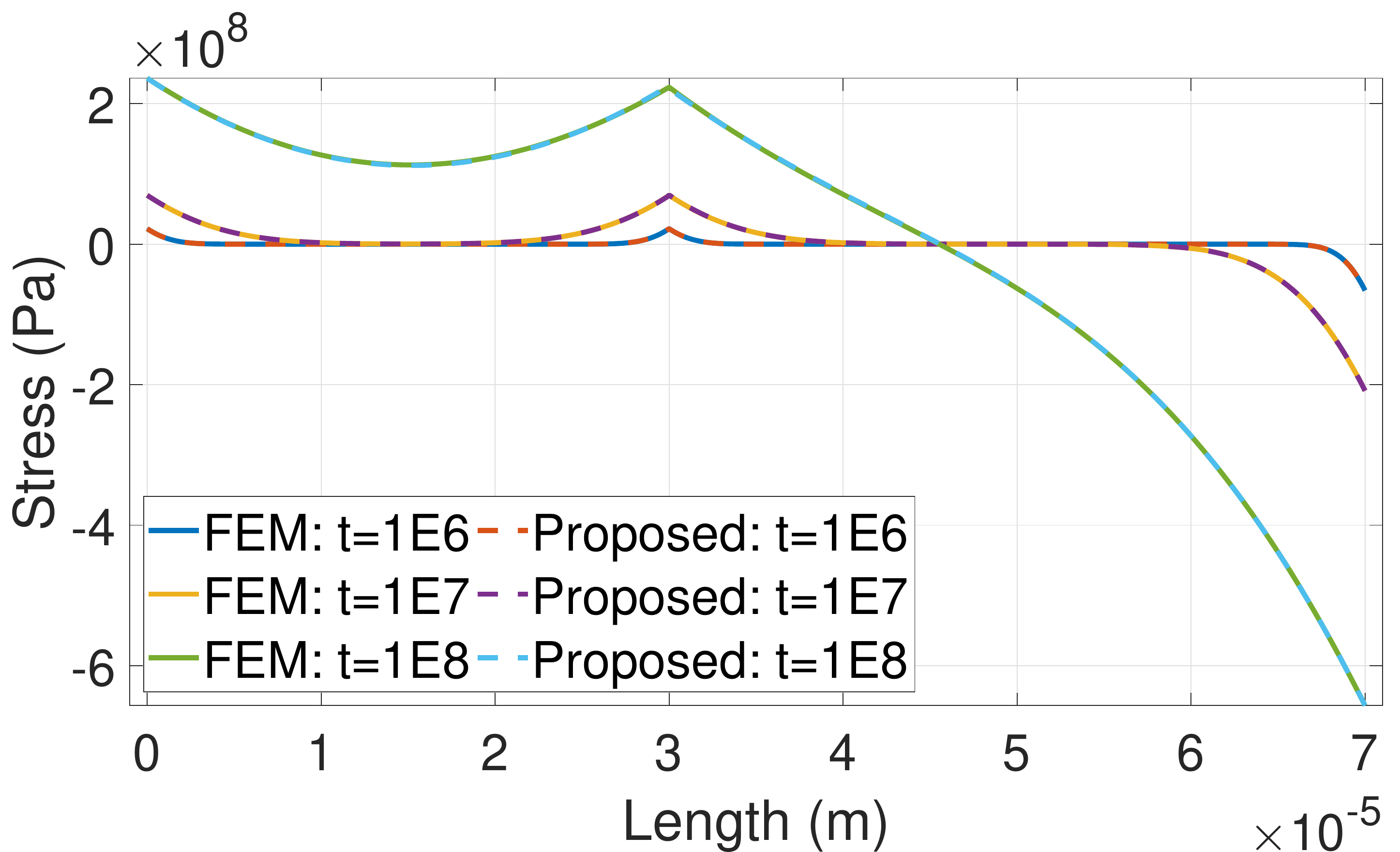}
	\label{fig:pic4}}
	\caption{{\color{black}Stress evolution results of two-segment straight wires: (a) $L_1=20\ \mu m,L_2=30\ \mu m,j_1 = -1\times 10^{10}\ A/m^2,j_2=4\times 10^{10}\ A/m^2$; (b) $L_1=40\ \mu m,L_2=40\ \mu m,j_1 = -2\times 10^{10}\ A/m^2,j_2=3\times 10^{10}\ A/m^2$; (c) $L_1=30\ \mu m,L_2=50\ \mu m,j_1 = -5\times 10^{10}\ A/m^2,j_2=1\times 10^{10}\ A/m^2$; (d) $L_1=30\ \mu m,L_2=40\ \mu m,j_1 = 0.3\times 10^{10}\ A/m^2,j_2=0.9\times 10^{10}\ A/m^2$.}}
\label{fig:parapic}
\end{figure}

}



\section{Conclusion}\label{conclusion}
This work targets the EM reliability problem and proposes a method for obtaining the stress evolution of complex interconnect trees during the void nucleation phase under time-varying temperature. Using multilayer perceptron and a customized objective function, we construct a trial function as the stress prediction expression to solve the physics-based constrained problem and reduce the required training data in the stress modeling. The proposed method reduces the training time compared with the competing learning-based method. We also discuss the importance of considering different widths for each interconnect segment rather than assuming an equal width. Experimental results demonstrate that the proposed method shows significant computational savings over competing schemes with high accuracy. {\color{black}The proposed work focuses on analyzing the EM-induced stress development during the pre-void phase, which is one of the complex EM failure processes of interconnects. By modifying the trial function and customizing a specific objective function for neural network training, we would like to extend our proposed framework to cover follow-up works of EM assessment such as the void growth phase analysis.}




\bibliographystyle{IEEEtran}
\bibliography{reliability}
\section*{Appendix}


\newcounter{mytempeqncnt}
\begin{figure*}[!t]
\normalsize
\setcounter{mytempeqncnt}{\value{equation}}
\setcounter{equation}{28}
\begin{equation}\label{eq:linear4}\footnotesize
\left[
\begin{array}{cccccccc}
ae^{-aL_{i_1}}&-ae^{aL_{i_1}}&0&0&0&0&0&0  \\
0&0&ae^{aL_{i_2}}&-ae^{-aL_{i_2}}&0&0&0&0  \\
0&0&0&0&ae^{-aL_{i_3}}&-ae^{aL_{i_3}}&0&0  \\
0&0&0&0&0&0&ae^{aL_{i_4}}&-ae^{-aL_{i_4}} \\
1&1&-1&-1&0&0&0&0\\
0&0&1&1&-1&-1&0&0\\
0&0&0&0&1&1&-1&-1\\
\kappa a&-\kappa a&-\kappa a&\kappa a&\kappa a&-\kappa a&-\kappa a&\kappa a
\end{array}
\right]
\cdot
\left[
\begin{array}{cccccccc}
A_1  \\
B_1 \\
A_2 \\
B_2 \\
A_3\\
B_3\\
A_4\\
B_4\\
\end{array}
\right]=
\left[
\begin{array}{cccccccc}
K_{i_1}^-  \\
K_{i_2}^+ \\
K_{i_3}^-\\
K_{i_4}^+ \\
0\\
0\\
0\\
-\kappa c_1+\kappa c_2-\kappa c_3 +\kappa c_4\\
\end{array}
\right].
\end{equation}
\begin{equation}\label{eq:ab}\footnotesize
	\begin{aligned}
	A_1=&\frac{-(w_{i_1}+w_{i_2}+w_{i_3}+w_{i_4})K_{i_1}^-e^{-3aL_{i_1}}+(-w_{i_1}c_1+w_{i_2}c_2-w_{i_3}c_3+w_{i_4}c_4)e^{-2aL_{i_1}}-(w_{i_1}c_1-w_{i_2}c_2+w_{i_3}c_3-w_{i_4}c_4)}{(w_{i_1}+w_{i_2}+w_{i_3}+w_{i_4})a(1-e^{-aL_{i_1}})}\\&+\frac{((-w_{i_1}+w_{i_2}+w_{i_3}+w_{i_4})K_{i_1}^-+2w_{i_2}K_{i_2}^+-2w_{i_3}K_{i_3}^-+2w_{i_4}K_{i_4}^+)e^{-aL_{i_1}}}{(w_{i_1}+w_{i_2}+w_{i_3}+w_{i_4})a(1-e^{-aL_{i_1}})},\\
	B_1=&\frac{-(w_{i_1}c_1-w_{i_2}c_2+w_{i_3}c_3-w_{i_4}c_4)e^{-4aL_{i_1}}+(-w_{i_1}c_1+w_{i_2}c_2-w_{i_3}c_3+w_{i_4}c_4)e^{-2aL_{i_1}}-(w_{i_1}+w_{i_2}+w_{i_3}+w_{i_4})K_{i_1}^-e^{-aL_{i_1}}}{(w_{i_1}+w_{i_2}+w_{i_3}+w_{i_4})a(1-e^{-aL_{i_1}})}\\&+\frac{((-w_{i_1}+w_{i_2}+w_{i_3}+w_{i_4})K_{i_1}^-+2w_{i_2}K_{i_2}^+-2w_{i_3}K_{i_3}^-+2w_{i_4}K_{i_4}^+)e^{-3aL_{i_1}}}{(w_{i_1}+w_{i_2}+w_{i_3}+w_{i_4})a(1-e^{-aL_{i_1}})},\\
	A_2=&\frac{-(w_{i_1}c_1-w_{i_2}c_2+w_{i_3}c_3-w_{i_4}c_4)e^{-4aL_{i_2}}+(-w_{i_1}c_1+w_{i_2}c_2-w_{i_3}c_3+w_{i_4}c_4)e^{-2aL_{i_2}}+(w_{i_1}+w_{i_2}+w_{i_3}+w_{i_4})K_{i_2}^+e^{-aL_{i_2}}}{(w_{i_1}+w_{i_2}+w_{i_3}+w_{i_4})a(1-e^{-aL_{i_2}})}\\&-\frac{(2w_{i_1}K_{i_1}^-+(w_{i_1}-w_{i_2}+w_{i_3}+w_{i_4})K_{i_2}^++2w_{i_3}K_{i_3}^--2w_{i_4}K_{i_4}^+)e^{-3aL_{i_2}}}{(w_{i_1}+w_{i_2}+w_{i_3}+w_{i_4})a(1-e^{-aL_{i_2}})},\\
	B_2=&\frac{(w_{i_1}+w_{i_2}+w_{i_3}+w_{i_4})K_{i_2}^+e^{-3aL_{i_2}}+(-w_{i_1}c_1+w_{i_2}c_2-w_{i_3}c_3+w_{i_4}c_4)e^{-2aL_{i_2}}-(w_{i_1}c_1-w_{i_2}c_2+w_{i_3}c_3-w_{i_4}c_4)}{(w_{i_1}+w_{i_2}+w_{i_3}+w_{i_4})a(1-e^{-aL_{i_2}})}\\&-\frac{(2w_{i_1}K_{i_1}^-+(w_{i_1}-w_{i_2}+w_{i_3}+w_{i_4})K_{i_2}^++2w_{i_3}K_{i_3}^--2w_{i_4}K_{i_4}^+)e^{-aL_{i_2}}}{(w_{i_1}+w_{i_2}+w_{i_3}+w_{i_4})a(1-e^{-aL_{i_2}})},\\
	A_3=&\frac{-(w_{i_1}+w_{i_2}+w_{i_3}+w_{i_4})K_{i_3}^-e^{-3aL_{i_3}}+(-w_{i_1}c_1+w_{i_2}c_2-w_{i_3}c_3+w_{i_4}c_4)e^{-2aL_{i_3}}-(w_{i_1}c_1-w_{i_2}c_2+w_{i_3}c_3-w_{i_4}c_4)}{(w_{i_1}+w_{i_2}+w_{i_3}+w_{i_4})a(1-e^{-4aL_{i_3}})}\\&+\frac{(-2w_{i_1}K_{i_1}^-+2w_{i_2}K_{i_2}^++(w_{i_1}+w_{i_2}-w_{i_3}+w_{i_4})K_{i_3}^-+2w_{i_4}K_{i_4}^+)e^{-aL_{i_3}}}{(w_{i_1}+w_{i_2}+w_{i_3}+w_{i_4})a(1-e^{-4aL_{i_3}})},\\
	B_3=&\frac{-(w_{i_1}c_1-w_{i_2}c_2+w_{i_3}c_3-w_{i_4}c_4)e^{-4aL_{i_3}}+(-w_{i_1}c_1+w_{i_2}c_2-w_{i_3}c_3+w_{i_4}c_4)e^{-2aL_{i_3}}-(w_{i_1}+w_{i_2}+w_{i_3}+w_{i_4})K_{i_3}^-e^{-aL_{i_3}}}{(w_{i_1}+w_{i_2}+w_{i_3}+w_{i_4})a(1-e^{-4aL_{i_3}})}\\&+\frac{(-2w_{i_1}K_{i_1}^-+2w_{i_2}K_{i_2}^++(w_{i_1}+w_{i_2}-w_{i_3}+w_{i_4})K_{i_3}^-+2w_{i_4}K_{i_4}^+)e^{-3aL_{i_3}}}{(w_{i_1}+w_{i_2}+w_{i_3}+w_{i_4})a(1-e^{-4aL_{i_3}})},\\	
	A_4=&\frac{-(w_{i_1}c_1-w_{i_2}c_2+w_{i_3}c_3-w_{i_4}c_4)e^{-4aL_{i_4}}+(-w_{i_1}c_1+w_{i_2}c_2-w_{i_3}c_3+w_{i_4}c_4)e^{-2aL_{i_4}}+(w_{i_1}+w_{i_2}+w_{i_3}+w_{i_4})K_{i_4}^+e^{-aL_{i_4}}}{(w_{i_1}+w_{i_2}+w_{i_3}+w_{i_4})a(1-e^{-4aL_{i_4}})}\\&-\frac{(2w_{i_1}K_{i_1}^--2w_{i_2}K_{i_2}^++2w_{i_3}K_{i_3}^-+(w_{i_1}+w_{i_2}+w_{i_3}-w_{i_4})K_{i_4}^+)e^{-3aL_{i_4}}}{(w_{i_1}+w_{i_2}+w_{i_3}+w_{i_4})a(1-e^{-4aL_{i_4}})},\\
	B_4=&\frac{(w_{i_1}+w_{i_2}+w_{i_3}+w_{i_4})K_{i_4}^+e^{-3aL_{i_4}}+(-w_{i_1}c_1+w_{i_2}c_2-w_{i_3}c_3+w_{i_4}c_4)e^{-2aL_{i_4}}-(w_{i_1}c_1-w_{i_2}c_2+w_{i_3}c_3-w_{i_4}c_4)}{(w_{i_1}+w_{i_2}+w_{i_3}+w_{i_4})a(1-e^{-4aL_{i_4}})}\\&-\frac{(2w_{i_1}K_{i_1}^--2w_{i_2}K_{i_2}^++2w_{i_3}K_{i_3}^-+(w_{i_1}+w_{i_2}+w_{i_3}-w_{i_4})K_{i_4}^+)e^{-aL_{i_4}}}{(w_{i_1}+w_{i_2}+w_{i_3}+w_{i_4})a(1-e^{-4aL_{i_4}})},\\
	\end{aligned}
\end{equation}   
\begin{equation}\label{eq:K}\footnotesize
	\begin{aligned}
	&\mathcal{L}(\frac{\partial\sigma_{i_1}}{\partial x_1}\Big|_{x_1=0})=\frac{-w_{i_1}c_1+w_{i_2}c_2-w_{i_3}c_3+w_{i_4}c_4}{w_{i_1}+w_{i_2}+w_{i_3}+w_{i_4}}+\frac{(2(w_{i_2}+w_{i_3}+w_{i_4})K_{i_1}^-+2w_{i_2}K_{i_2}^+-2w_{i_3}K_{i_3}^-+2w_{i_4}K_{i_4}^+)(e^{-aL_{i_1}}-e^{-3aL_{i_1}})}{(w_{i_1}+w_{i_2}+w_{i_3}+w_{i_4})(1-e^{-4aL_{i_1}})},\\
	&\mathcal{L}(\frac{\partial\sigma_{i_2}}{\partial x_2}\Big|_{x_2=0})=\frac{w_{i_1}c_1-w_{i_2}c_2+w_{i_3}c_3-w_{i_4}c_4}{w_{i_1}+w_{i_2}+w_{i_3}+w_{i_4}}+\frac{(2w_{i_1}K_{i_1}^-+2(w_{i_1}+w_{i_3}+w_{i_4})K_{i_2}^++2w_{i_3}K_{i_3}^--2w_{i_4}K_{i_4}^+)(e^{-aL_{i_2}}-e^{-3aL_{i_2}})}{(w_{i_1}+w_{i_2}+w_{i_3}+w_{i_4})(1-e^{-4aL_{i_2}})},\\
	&\mathcal{L}(\frac{\partial\sigma_{i_3}}{\partial x_3}\Big|_{x_3=0})=\frac{-w_{i_1}c_1+w_{i_2}c_2-w_{i_3}c_3+w_{i_4}c_4}{w_{i_1}+w_{i_2}+w_{i_3}+w_{i_4}}+\frac{(-2w_{i_1}K_{i_1}^-+2w_{i_2}K_{i_2}^++2(w_{i_1}+w_{i_2}+w_{i_4})K_{i_3}^-+2w_{i_4}K_{i_4}^+)(e^{-aL_{i_3}}-e^{-3aL_{i_3}})}{(w_{i_1}+w_{i_2}+w_{i_3}+w_{i_4})(1-e^{-4aL_{i_3}})},\\
	&\mathcal{L}(\frac{\partial\sigma_{i_4}}{\partial x_4}\Big|_{x_4=0})=\frac{w_{i_1}c_1-w_{i_2}c_2+w_{i_3}c_3-w_{i_4}c_4}{w_{i_1}+w_{i_2}+w_{i_3}+w_{i_4}}+\frac{(2w_{i_1}K_{i_1}^--2w_{i_2}K_{i_2}^++2w_{i_3}K_{i_3}^-+2(w_{i_1}+w_{i_2}+w_{i_3})K_{i_4}^+)(e^{-aL_{i_4}}-e^{-3aL_{i_4}})}{(w_{i_1}+w_{i_2}+w_{i_3}+w_{i_4})(1-e^{-4aL_{i_4}})}.\\
	\end{aligned}
\end{equation}   
\setcounter{equation}{\value{mytempeqncnt}}
\hrulefill
\vspace*{4pt}
\end{figure*}

\subsection{Derivation of the trial function}\label{trail}
We employ the Laplace transformation technique and use $\hat{\varPsi}_t(x,s)=\mathcal{L}(\varPsi_t(x,t))=\int^{+\infty}_0e^{-st}\varPsi_t(x,t)dt$ to represent the Laplace form of the trial function, so that the diffusion constraint of \eqref{eq:Korhonen's PDE} is converted to an ordinary differential equation

\begin{equation}\label{eq:ode}
\frac{d^2\hat{\varPsi}_t(x,s)}{dx^2}-\frac{s}{\kappa}\hat{\varPsi}_t(x,s)=0,0<x<L.\\
\end{equation}
Based on the characteristic equation method, the general solution of second order homogeneous linear constant equations and IC in \eqref{eq:ic}, we can obtain $\hat{\varPsi}_t(x,s)$ by
\begin{equation} \label{eq:hatsigma}
\hat{\varPsi}_t(x,s)=Ae^{\sqrt{\frac{s}{\kappa}}x}+Be^{-\sqrt{\frac{s}{\kappa}}x},0<x<L.\\
\end{equation}
Here, the coefficients $A,\ B$ are determined by BC in \eqref{eq:Boundary} and we define $k^+(t)=k(t,\theta^+),\ k^-(t)=k(t,\theta^-)$.
Substituting $\hat{\varPsi}_t(x,s)$ from \eqref{eq:hatsigma} into the Laplace form of \eqref{eq:Boundary}, we obtain the following linear well-posed equation
\begin{equation}\label{eq:linearSyst}
\left[
\begin{array}{cccccc}
a&-a  \\
ae^{aL}&-ae^{-aL}  \\
\end{array}
\right]
\cdot
\left[
\begin{array}{cccccc}
A  \\
B \\
\end{array}
\right]=
\left[
\begin{array}{cccccc}
\frac{D^-(s)+k^+(0)}{s}  \\
\frac{D^+(s)+k^-(0)}{s}  \\
\end{array}
\right],
\end{equation}
where $a=\sqrt{s/ \kappa},\ D^-(s)=sK^-(s)-k^-(0),\ D^+(s)=sK^+(s)-k^+(0)$. The notations $K^-(s),\ K^+(s)$ represent the Laplace form of $k^-(t),\ k^+(t)$. Thus, solving the linear system \eqref{eq:linearSyst} yields
\begin{equation}\label{eq:a1b1} \small
\begin{aligned}
A=&\!-\!\frac{D\!^-(s)\!+\!k\!^-(0)}{sa}\frac{e^{-2aL}}{1\!-\!e^{-2aL}}+\frac{D\!^+(s)\!+\!k\!^+(0)}{sa}\frac{e^{-aL}}{1\!-\!e^{-2aL}},\\
B=&\!-\!\frac{D\!^-(s)\!+\!k\!^-(0)}{sa}\frac{1}{1\!-\!e^{-2aL}}+\frac{D\!^+(s)\!+\!k\!^+(0)}{sa}\frac{e^{-aL}}{1\!-\!e^{-2aL}}.
\end{aligned}
\end{equation}

We then employ the complementary error function to construct a basis function $g(x,t)$ in \eqref{eq:basisFunc}. 
In particular, the complementary error function is widely used in the digital communication system, heat equation, etc. 
Coupling spatial functions $\xi_q(n,x,L)(q=1,2,3,4)$ in \eqref{eq:notation} and the basis function $g(x,t)$ in \eqref{eq:basisFunc}, we can obtain the $\varPsi_t(x,t)$ by the inverse Laplace transformation on \eqref{eq:hatsigma} with the known coefficients in \eqref{eq:a1b1}
\begin{equation}\label{eq:convinfty}
	\begin{aligned}
	&\varPsi_t(x,t,L,\theta^-,\theta^+)=\\
	&\sum_{n=0}^{+\infty}\Big(\frac{-dk(t,\theta^-)}{dt}\ast \big(g(\xi_1(n,x,L),t)+g(\xi_3(n,x,L),t)\big)\\
	&-k(0,\theta^-)\times\big(g(\xi_1(n,x,L),t)+g(\xi_3(n,x,L),t)\big)
	\\&+\frac{dk(t,\theta^+)}{dt}\ast \big(g(\xi_2(n,x,L),t)+g(\xi_4(n,x,L),t)\big)\\
	&+k(0,\theta^+)\times \big(g(\xi_2(n,x,L),t)+g(\xi_4(n,x,L),t)\big)\Big).\\
	\end{aligned}
\end{equation}


\subsection{Proof of Theorem \ref{theory1}}\label{proof}
We first write the BCs for stress evolution of the four segments as follows
\begin{equation}\label{eq:restrictions}
    \begin{aligned}
    &\kappa_{i_1}\Big(\frac{\partial\sigma_{i_1}}{\partial x_1}-k_{i_1}^-(t)\Big)=0,x_1=-L_{i_1},t>0\\
    &\kappa_{i_2}\Big(\frac{\partial\sigma_{i_2}}{\partial x_2}-k_{i_2}^+(t)\Big)=0,x_2=L_{i_2},t>0\\
    &\kappa_{i_3}\Big(\frac{\partial\sigma_{i_3}}{\partial x_3}-k_{i_3}^-(t)\Big)=0,x_3=-L_{i_3},t>0\\
    &\kappa_{i_4}\Big(\frac{\partial\sigma_{i_4}}{\partial x_4}-k_{i_4}^+(t)\Big)=0,x_4=L_{i_4},t>0\\
    &\sigma_{i_1}=\sigma_{i_2}=\sigma_{i_3}=\sigma_{i_4},x_1=x_2=x_3=x_4=0,t>0\\
    &w_{i_1}\kappa_{i_1}\Big(\frac{\partial\sigma_{i_1}}{\partial x_1}+G_{i_1}\Big)-w_{i_2}\kappa_{i_2}\Big(\frac{\partial\sigma_{i_2}}{\partial x_2}+G_{i_2}\Big)\\
    &\qquad+w_{i_3}\kappa_{i_3}\Big(\frac{\partial\sigma_{i_3}}{\partial x_3}+G_{i_3}\Big)-w_{i_4}\kappa_{i_4}\Big(\frac{\partial\sigma_{i_4}}{\partial x_4}+G_{i_4}\Big)=0,\\
    &\qquad\qquad\qquad\qquad\qquad x_1=x_2=x_3=x_4=0,t>0
    \end{aligned}
\end{equation}
We then employ the Laplace transformation technique and construct $\mathcal{L}(\sigma_{i_m})=A_me^{\sqrt{s/\kappa}x}+B_me^{-\sqrt{s/\kappa}x}$ respecting Korhonen's equation. The Laplace form of BCs in \eqref{eq:restrictions} yields the linear system in \eqref{eq:linear4}, where $\kappa_{i_1}=\kappa_{i_2}=\kappa_{i_3}=\kappa_{i_4}=\kappa,\ a=\sqrt{s/\kappa}$ and $ c_m=G_{i_m}/s$. The coefficients $A_m,B_m$ are given by \eqref{eq:ab}. Then we substitute \eqref{eq:ab} into the stress gradients at the center node by $\mathcal{L}(\partial \sigma_{i_m}/\partial x_m|_{x_m=0})=aA_m-aB_m$. The stress gradients in Laplace form are shown in \eqref{eq:K}.
The initial stress gradient at the center node can be obtained by employing the initial value theorem of Laplace transformation on \eqref{eq:K}. Moreover, it is known in BCs that the stress gradient satisfies $\partial \sigma_b/\partial x|_{x=x_b}=-G_b$ at the terminal. Finally, the initial stress gradient at nodes of interconnect tree follows \eqref{eq:initial gradient}.



\vspace{-10 mm}

\begin{IEEEbiography}
[{\includegraphics[width=1in,height=1.25in,clip,keepaspectratio]{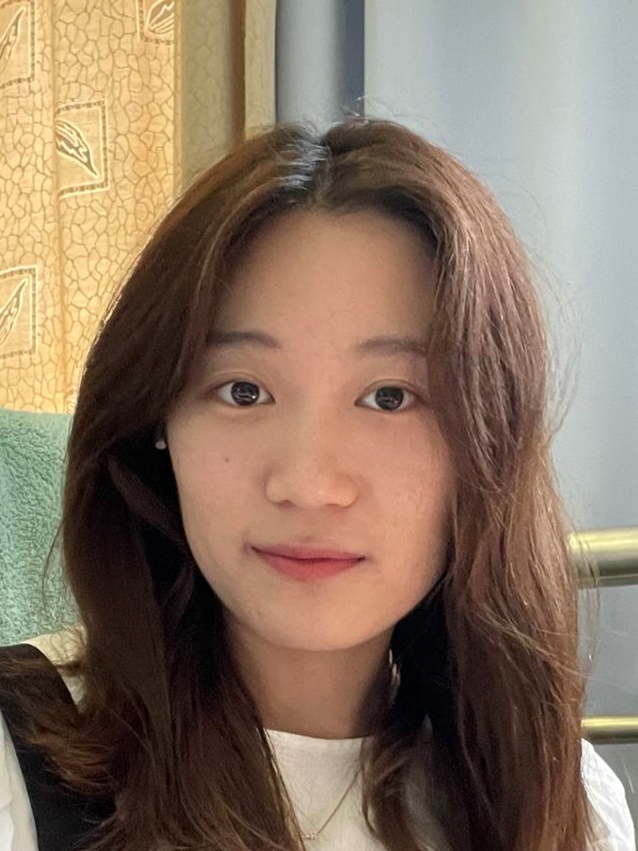}}]
{Tianshu Hou} received the B.S. degree in electronic information science and technology from Sichuan University, Sichuan, China in 2019. She is currently pursuing a Ph.D degree in the Department of Micro/Nano-electronics, Shanghai Jiao Tong University, Shanghai, China. Her current research interests include electromigration reliability modeling, assessment and optimization.
\end{IEEEbiography}

\vspace{-15 mm}
\begin{IEEEbiography}
[{\includegraphics[width=1in,height=1.2in,clip,keepaspectratio]{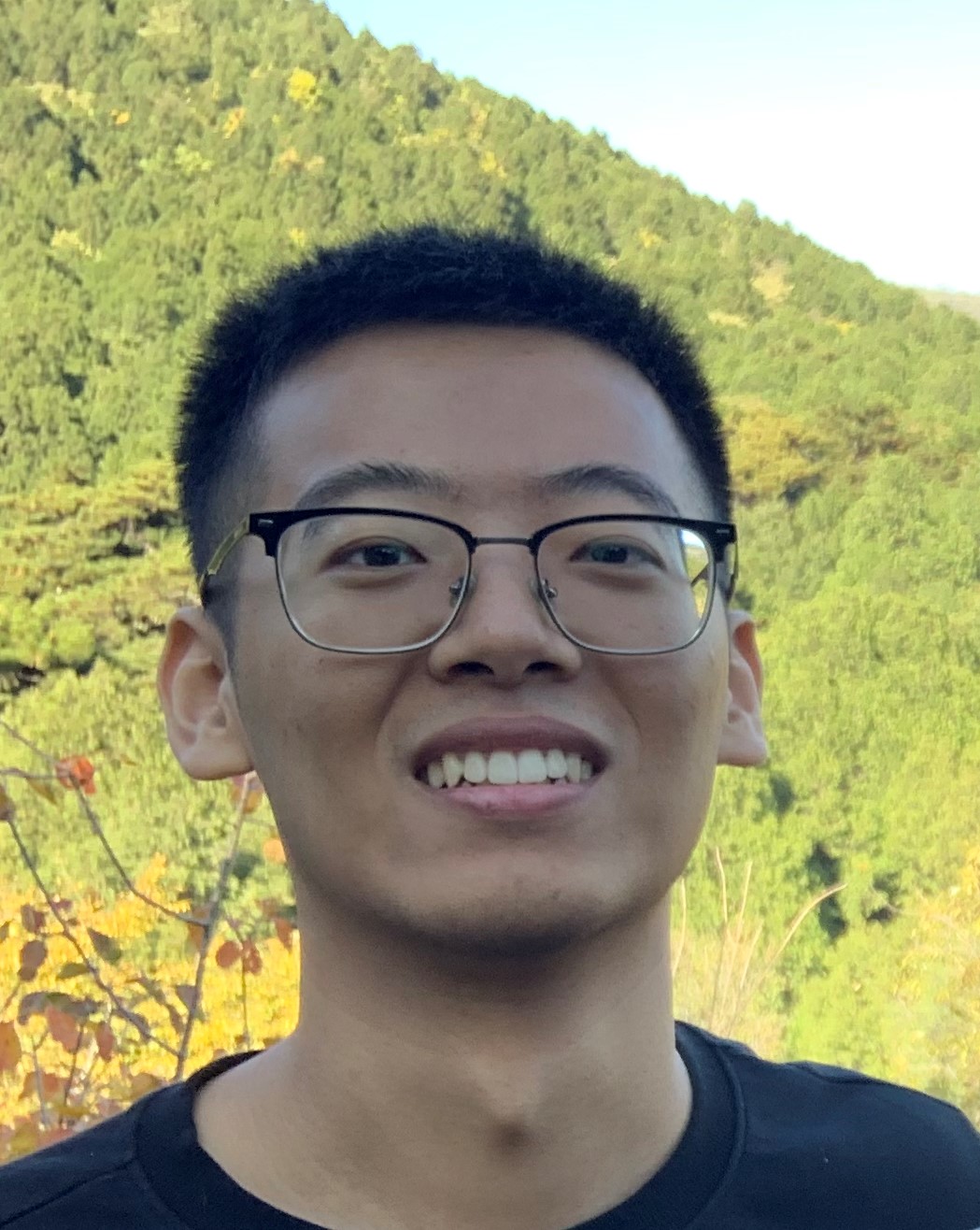}}]
{Peining Zhen}received the B.S. degree in electronic engineering from Sichuan University, Chengdu, China, in 2017. He is currently pursuing the Ph.D degree with the Department of Micro/Nano-Electronics, Shanghai Jiao Tong University, Shanghai, China. His current research interests include machine learning and neuromorphic computing.  
\end{IEEEbiography}

\vspace{-15 mm}
\begin{IEEEbiography}
[{\includegraphics[width=1in,height=1.25in,clip,keepaspectratio]{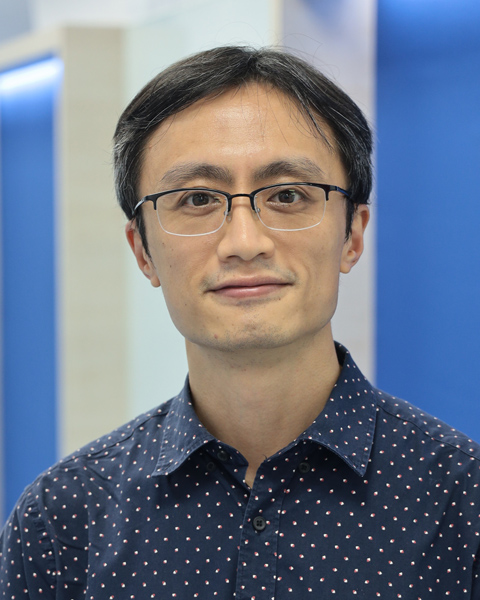}}]{Ngai Wong}  (SM, IEEE) received his B.Eng in 1999 and Ph.D. in EEE from The University of Hong Kong (HKU) in 2003, and he was a visiting scholar with Purdue University, West Lafayette, IN, in 2003. He is currently an Associate Professor with the Department of Electrical and Electronic Engineering at HKU. His research interests include electronic design automation (EDA), model order reduction, tensor algebra, linear and nonlinear modeling \& simulation, and compact neural network design. 
\end{IEEEbiography}

\vspace{-15 mm}

\begin{IEEEbiography}
[{\includegraphics[width=1in,height=1.25in,clip,keepaspectratio]{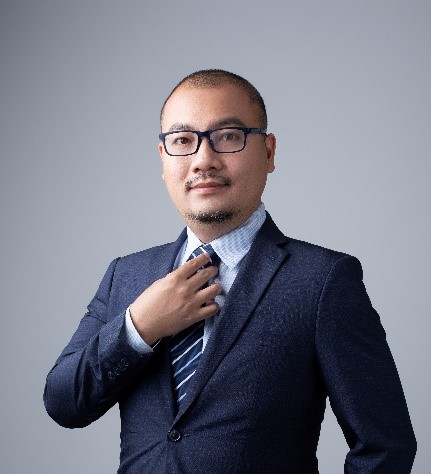}}]{Quan Chen} (S'09-M'11) received his B.S. degree in Electrical Engineering from the Sun Yat-Sen University, China, in 2005 and the M.Phil. and Ph.D. degree in Electronic Engineering from The University of Hong Kong, Hong Kong, in 2007 and 2010. From 2010-2011 he was postdoctoral fellow at the department of Computer Science and Engineering of the University of California, San Diego (UCSD). In 2012-2018, he was a research assistant professor at the department of Electrical and Electronic Engineering, The University of Hong Kong (HKU). He joined the Southern University of Science and Technology (SUSTech) in Shenzhen, China in 2019, where he is an assistant professor now. 
His research interests include ultra-large-scale circuit simulation and multi-physics analysis in the field of electronic design automation (EDA), as well as EDA techniques for emerging technologies such as sub-10nm devices, memristors, and quantum computing. He also has years of experience in technical transformation and commercialization.

\end{IEEEbiography}

\vspace{-15 mm}

\begin{IEEEbiography}
[{\includegraphics[width=1in,height=1.25in,clip,keepaspectratio]{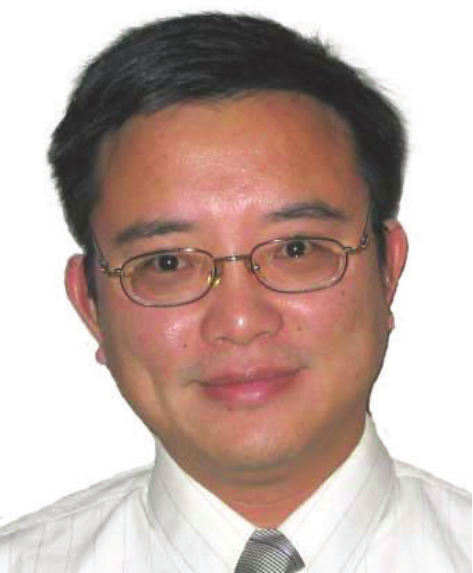}}]{Guoyong Shi} (S'99-M'02-SM'11)
received the B.S. in applied mathematics from Fudan University, Shanghai, China, the M.S. degree in electronics and information science from Kyoto Institute of Technology, Kyoto, Japan, and the Ph.D. degree in electrical engineering from Washington State University, Pullman, in 1987, 1997, and 2002, respectively.

He is now a Professor of Microelectronics in Shanghai Jiao Tong University in Shanghai, China.
His research interests include design automation of analog/mixed-signal integrated circuits and systems.
He has published about 100 research papers in technical journals and conferences.
He is co-author of the book \emph{Advanced Symbolic Analysis for VLSI Systems --
Methods and Applications} published by Springer in 2014.
He has served several technical program committees including ASPDAC and SMACD.
He currently serves on the editorial board of Integration, the VLSI journal.
Dr. Shi was co-recipient of the Donald O. Pederson Best Paper Award in 2007.
\end{IEEEbiography}

\vspace{-10 mm}

\begin{IEEEbiography}
[{\includegraphics[width=1in,height=1.25in,clip,keepaspectratio]{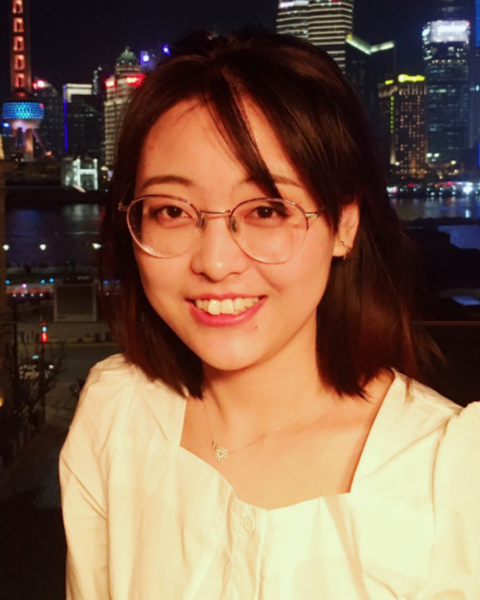}}]
{Shuqi Wang} received the B.Eng. degree in  Microelectronics Science and Engineering from University of Electronic Science and Technology of China, Sichuan, China in 2020. She is currently pursuing a master degree in the Department of Micro/Nano-electronics, Shanghai Jiao Tong University, Shanghai, China. Her research interests include machine learning and neuromorphic computing.
\end{IEEEbiography}

\vspace{-15 mm}

\begin{IEEEbiography}
[{\includegraphics[width=1in,height=1.25in,clip,keepaspectratio]{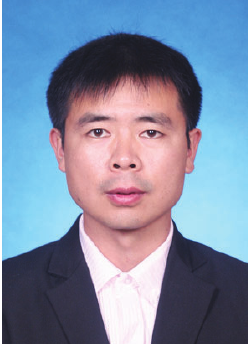}}]
{Hai-Bao Chen} received the B.S. degree in information and computing sciences, and the M.S. and Ph.D. degrees in applied mathematics from Xian Jiaotong University, Xian, China, in 2006, 2008, and 2012, respectively. He then joined Huawei Technologies, where he focused on cloud computing and big data. He was a Post-Doctoral Research Fellow with Electrical Engineering Department, University of California, Riverside, CA, USA, from 2013 to 2014. He is currently an Associate Professor in the Department of Micro/Nano-electronics, Shanghai Jiao Tong University, Shanghai, China. His current research interests include VLSI reliability, machine learning and neuromorphic computing, numerical analysis and modeling for VLSIs, integrated circuit for signal and control systems. Dr. Chen has authored or co-authored about 70 papers in scientific journals and conference proceedings. He received one Best Paper Award nomination from Asia and South Pacific Design Automation Conference (ASP-DAC) in 2015. Since 2016, Dr. Chen serves as an Associate Editor for Integration-the VLSI Journal.
\end{IEEEbiography}

\end{document}